\documentclass[twoside]{article}

\usepackage[preprint]{aistats2027}
\usepackage{graphicx} 
\usepackage[english]{babel}
\usepackage[nopatch=footnote]{microtype}      
\usepackage{amssymb,dsfont,amsmath,amsthm,mathtools,mathrsfs}
\usepackage{mathtools}
\usepackage{graphicx}
\usepackage{tikz-cd}
\usepackage{pb-diagram}
\usepackage{comment}
\usepackage{tcolorbox}
\usepackage{enumitem}
\usepackage{hyperref}
\usepackage{url}

\let\originalleft\left
\let\originalright\right
\renewcommand{\left}{\mathopen{}\mathclose\bgroup\originalleft}
\renewcommand{\right}{\aftergroup\egroup\originalright}

\makeatletter
\newcommand{\myitem}[1]{%
	\item[#1]\protected@edef\@currentlabel{#1}%
}
\makeatother

\newcommand{\sg}{\operatorname{SG}}

\theoremstyle{plain}
\newtheorem{theorem}{Theorem}[section]

\newtheorem{lemma}[theorem]{Lemma}
\newtheorem{proposition}[theorem]{Proposition}
\newtheorem{corollary}[theorem]{Corollary}

\newtheoremstyle{my_def}
  {\topsep}   
  {\topsep}   
  {\normalfont}  
  {0pt}       
  {\bfseries\itshape} 
  {.}         
  {5pt plus 1pt minus 1pt} 
  {}          
\theoremstyle{my_def}

\newtheoremstyle{my_rem}
  {1.0\topsep}   
  {1.0\topsep}   
  {\normalfont}  
  {0pt}       
  {\bfseries\itshape} 
  {.}         
  {5pt plus 1pt minus 1pt} 
  {}          
\theoremstyle{my_rem}
\newtheorem{remark}[theorem]{Remark}

\numberwithin{equation}{section}
\numberwithin{theorem}{section}

\newcommand{\R}{\mathbb{R}}
\newcommand{\RR}{\mathbb{R}}

\newcommand{\cP}{\mathcal{P}}

\newcommand{\sfN}{\mathsf{N}}
\newcommand{\sfD}{\mathsf{D}}
\newcommand{\al}{\alpha}

\renewcommand{\d}[1]{\mathrm{d}}
\DeclareMathOperator{\Att}{Att}

\renewcommand*\d{\mathop{}\!\mathrm{d}}

\newcommand{\slot}{{\,\cdot\,}}
\newcommand{\push}{_{\#}}
\newcommand{\set}[2]{{\left\{ #1 \,\middle|\, #2 \right\}}}

\newcommand{\sfW}{\mathsf{W}}
\newcommand{\sfT}{\mathsf{T}}
\newcommand{\N}{\mathbb{N}}

\newcommand{\defeq}{\coloneqq}
\newcommand{\eqdef}{\eqqcolon}

\renewcommand{\eqref}[1]{\textup{(\ref{#1})}}
\newcommand{\E}{\mathbb{E}}

\DeclarePairedDelimiterX{\iptemp}[2]{\langle}{\rangle}{#1, #2}
\newcommand{\ip}{\iptemp}
\DeclarePairedDelimiterX{\normtemp}[1]{\lVert}{\rVert}{#1}
\newcommand{\norm}{\normtemp}
\DeclarePairedDelimiterX{\abstemp}[1]{\lvert}{\rvert}{#1}
\newcommand{\abs}{\abstemp}
\DeclarePairedDelimiterX{\trtemp}[1]{(}{)}{#1}

\DeclarePairedDelimiterX{\SEtemp}[2]{(}{)}{#1, #2}

\DeclarePairedDelimiterX{\SGtemp}[1]{(}{)}{#1}

\def\qfa{\quad\text{for all}\quad}

\def\qa{\quad\text{and}\quad}
\def\qf{\quad\text{for}\quad}

\usepackage{xcolor}
\usepackage[textsize=tiny]{todonotes}

\definecolor{darkred}{rgb}{.7,0,0}
\definecolor{darkgreen}{rgb}{0.2,0.5,0.3}
\definecolor{darkblue}{rgb}{0,0,0.7}
\definecolor{darkmagenta}{rgb}{0.55,0,0.55}
\definecolor{darkteal}{rgb}{0,0.45,0.55}
\definecolor{darkorange}{rgb}{1,0.40,0}

\newcommand{\mytitle}{
Stability of Measure-to-Measure Transformers on Sub-Gaussian Data
}

\usepackage[round,sort]{natbib}

\begin{document}

%

%

\twocolumn[

\aistatstitle{\mytitle}


 \aistatsauthor{ Frank Cole \And Nicholas H. Nelsen \And  Takashi Furuya }
 \aistatsaddress{ UCLA\\
 {\texttt{fcole99@g.ucla.edu}}
 \And  UT Austin\\
 {\texttt{nnelsen@oden.utexas.edu}}
 \And Doshisha University, RIKEN AIP\\
 {\texttt{tfuruya@mail.doshisha.ac.jp}}
 }

]

\begin{abstract}
Transformers have exhibited impressive empirical success across various domains, but their theoretical foundations remain less developed. This work constitutes a mathematical study of the measure-to-measure operators defined by transformers. We show that transformers map sub-Gaussian inputs to sub-Gaussian outputs; this ensures that taking arbitrary-length compositions of the softmax operator is well-defined. We then show that transformers are H\"older continuous with respect to the $1$-Wasserstein distance on appropriate spaces of sub-Gaussian inputs. This allows us to establish estimates on the error propagation along a transformer between a sub-Gaussian measure and its empirical approximation. We also study a mean-field analog of the cross-attention mechanism, which is an operator from a pair of probability measures to a single probability measure. We show that cross-attention exhibits different H\"older regularity and sample-complexity in its two input arguments. Last, we apply our results to deduce approximation guarantees for measure-to-measure transformers. Together, these results provide a firm stability and finite-sample theory for transformers on sub-Gaussian data.
\end{abstract}

\section{Introduction}
Transformers \citep{vaswani2017attention} are the backbone of the modern generative artificial intelligence (GenAI) revolution and the architecture behind large language models like ChatGPT \citep{achiam2023gpt}. They also play a burgeoning role in scientific discovery through the development of scientific foundation models \citep{bodnar2025foundation, takeda2023foundation, horawalavithana2022foundation,wiesner2025towards, subramanian2023towards,herde2024poseidon}.
While the scale and applications of GenAI continue to grow, our theoretical understanding of these models is far less mature. This paper falls within a recent research area which studies transformers from a mathematical perspective, with the hope of eventually narrowing the gap between theory and practice. By developing a rigorous mathematical theory of GenAI architectures, one hopes to reveal the structure behind both their successes and failures. One can also leverage these theoretical insights to design more reliable and robust algorithms for GenAI and machine learning more generally.

The analysis in this paper is based on the \textit{mean-field} or \textit{measure-theoretic} perspective on encoder-only transformers \citep{rigollet2026mean, castin2025unified, furuya2024transformers,sander2022sinkformers,geshkovski2025mathematical,vuckovic2020mathematical}: the output of a transformer is a function of the empirical probability distribution of its input tokens.
There are several motivations for studying transformers from this perspective. First, by lifting the analysis of transformers from a discrete to a continuum setting, we unlock powerful and well-developed mathematical tools for understanding the behavior of transformers and machine learning on spaces of probability measures more broadly. Second, high-dimensional data often admit natural interpretations as collections of probability measures. Measure-valued data arises, for instance, in physical sciences, where swarms of particles interact according to a time-varying probability distribution, or in neural surrogates for probabilistic models, where the goal is to train a network that directly maps problem conditions (e.g., prior distributions and observations) to solutions. Last, AI methods for such problems have achieved substantial empirical success, but lack the theoretical grounding necessary to engender reliability and trust.

While recent work in this area studies transformers as in-context or contextual mappings \citep{furuya2024transformers}---i.e., vector-to-vector mappings which depend on a probability measure or \textit{context}---our work focuses on the measure-to-measure operators defined by transformers \citep{furuya2025transformers,furuya2026universal,lavenant2026continuous,vandergrift2026measure,geshkovski2024measure}, which we refer to as \textit{mean-field transformer operators}.
The goal of this paper is to answer fundamental questions about the mathematical properties of these operators, particularly for input probability measures which have unbounded support. Prior theoretical results focused primarily on bounded data \citep{castin2024smooth, vuckovic2021regularity}, for which the answers to these questions differ substantially. To this end, sub-Gaussian distributions are useful for modeling distributions of tokens which occasionally take very large values, but otherwise remain bounded. Additionally, many probability distributions arising naturally in machine learning and scientific computing, such as posterior distributions for Bayesian inverse problems with Gaussian priors \citep{stuart2010inverse} and data distributions along the forward/backward processes of a score-based generative model \citep{song2020score}, often exhibit fast tail decay, but necessarily have unbounded support. At the same time, when working with distributions of unbounded support, some assumption of tail decay is necessary to ensure that the softmax operation is well-defined. Our work finds that sub-Gaussianity is a sufficient tail condition for the stability of mean-field transformer operators, while being broad enough not to preclude the examples of interest mentioned.

\subsection{Main Contributions}
Our main contributions are stated as follows.
\begin{enumerate}[label={(C\arabic*)},labelindent=0.05em,leftmargin=*]
	\item \textbf{Well-definedness.} We show that if $\mu$ is sub-Gaussian and $\mathsf{T}$ is an arbitrarily deep mean-field transformer, then  $\mathsf{T}(\mu,\slot)_{\#}\mu$ remains sub-Gaussian, ensuring that the softmax normalizing constant remains finite across layers. This result thus identifies sub-Gaussianity as a sufficient condition for the image of measures with unbounded support under mean-field transformers to be well-defined; see Section \ref{subsec-results-subG} for precise results. 
	\item \textbf{Regularity.} We show that the mapping $\mu \mapsto \mathsf{T}(\mu,\slot)_{\#}\mu$ is H\"older continuous on appropriate subsets of sub-Gaussian measures with respect to the $1$-Wasserstein metric; see Section \ref{subsec-results-holder} for precise results. 
	\item \textbf{Sample complexity.} We estimate the expected distance between a sub-Gaussian measure $\mu$ and its empirical counterpart $\mu_N$ under a mean-field transformer in the $1$-Wasserstein distance. Our estimate is of the form
	\begin{align*}
		\E\Bigl[\sfW_1\bigl(\mathsf{T}(\mu,\slot)_{\#}\mu,\mathsf{T}(\mu_N,\slot)_{\#}\mu_N\bigr)\Bigr] = \widetilde{O}\bigl(N^{-\gamma/d}\bigr),
	\end{align*}
	where $\gamma$ is the H\"older exponent of the mapping $\mu \mapsto \mathsf{T}(\mu,\slot)_{\#}\mu$; see Section \ref{subsec-results-samples} for precise results.
	As an auxiliary result which is possibly of independent interest, we prove concentration inequalities for the sub-Gaussian parameters of $\mu_N$ (Proposition~\ref{prop-subG-empirical}).
	\item \textbf{Cross-attention.} We also study mean-field transformers with the cross-attention mechanism, which leads to a mapping of the form $(\mu,\eta) \mapsto \mathsf{T}(\mu,\slot)_{\#}\eta.$ We show that mean-field cross-attention is jointly H\"older continuous, and that the different variables exhibit different regularity. We also prove a sample complexity result for this mapping under the assumption that $\mu$ is compactly supported and $\eta$ is sub-Gaussian; see Section \ref{subsec-results-cross} for precise results.
    \item \textbf{Measure-to-measure approximation.} By pairing our sample complexity estimates with recent approximation-theoretic results \citep{furuya2026universal}, we prove that mean-field transformer operators universally approximate continuous measure-to-measure operators on certain $\sfW_1$-compact subsets of $\mathcal{P}(\RR^d),$ even when the transformer only has access to samples from the input measure. Section~\ref{sec:main_approx} has the precise results.
\end{enumerate}
Our work thus provides initial understanding of the behavior of measure-to-measure transformers on unbounded data. The results in this paper may be of use in other applications, such as to study the generalization properties of transformer-based machine learning models on the space of probability measures.

\subsection{Related Work}
The mean-field perspective on transformers has received significant recent research attention.

\paragraph{Mean-field dynamics of transformers.} The works of \citet{geshkovski2023emergence, geshkovski2024measure,geshkovski2025mathematical} pioneered the interpretation of infinitely-deep transformers as interacting particle systems, where residual connections are viewed as steps of an Euler method. To better understand these dynamics as the number of tokens increases, several works have studied the corresponding \textit{mean-field dynamics} of transformers \citep{rigollet2026mean}. Recent works in mean-field dynamics of transformers study the emergence of clusters \citep{karagodin2024clustering, chen2025quantitative, bruno2025emergence,li2026diverse}, softmax temperature scaling \citep{bruno2026scaling, chen2025critical}, multi-scale analysis \citep{bruno2026multiscale}, generalizations to multi-head attention \citep{pendharkar2026gradient}, and the effects of non-compactly supported initial data \citep{castin2025unified}, among other topics.

\paragraph{Regularity of transformers.} 
Several works have established the Lipschitz regularity of mean-field self-attention \citep{kim2021lipschitz,castin2024smooth,vuckovic2021regularity}; these results show that (dot product) attention is not Lipschitz unless measures are restricted to have common compact support, or the negative squared distance is used in the place of the dot product to compute attention scores. The work of \citet{mroueh2023towards} studies the Lipschitz regularity of in-context learning in the mean-field setting, with transformers as a representative example. For different perspectives on the regularity of attention, see work of \citet{sander2022sinkformers}, which connects attention with Sinkhorn normalization to heat diffusion, \citet{kan2026optimal}, which interprets regularized transformers as solutions to optimal control problems, and \citet{kan2025stability}, which studies the Lipschitz stability of transformers induced by layer normalization. Closely related to our work is that of \citet{hargenrader2026attention}, who study the integrability of post-norm transformers (beyond softmax attention) when applied to heavy-tailed input measures; in particular, such measures have unbounded support.

\paragraph{Sample complexity of transformers.}
\Citet{bohbot2025token, boursier2025softmax, chen2026propagation} study the sample complexity of mean-field transformers. Their contributions share a similar flavor with our results in Section~\ref{subsec-results-samples}: in particular, \citet{bohbot2025token, boursier2025softmax} prove concentration inequalities for attention with sub-Gaussian inputs. The main difference is that their results study the sample complexity of the in-context map associated with transformers, i.e., the mapping $x \mapsto \mathsf{T}(\mu,x)$, whereas our results are concerned with the sample complexity of the measure-to-measure operators $\mu \mapsto \mathsf{T}(\mu,\slot)_{\#}\mu$ defined by transformers $\sfT$. Thus, the results of \citet{bohbot2025token, boursier2025softmax, chen2026propagation} are not directly comparable with our results.

\paragraph{Learning on spaces of probability measures.} 
Mean-field transformers fall within the wider area of machine learning on spaces of probability measures \citep{leon2026wasserstein,panaretos2020invitation,yao2026any,zaheer2017deep}. Applications include regression on point cloud data \citep{vandergrift2026measure,cai2020learning}, training distribution design \citep{guerra2025learning}, amortized solvers for mean-field games and optimal transport \citep{huang2025unsupervised, cole2026context}, amortized conditioning and generative modeling \citep{tsimpos2026one, gloeckler2024all, whittle2026distribution, fishman2026distribution, generalemodeling}, and inverse problems \citep{nelsen2026operator} such as non-Gaussian Bayesian data assimilation \citep{bach2025learning, bach2026learning, wang2026global,al2025fast}.

\subsection{Notation}
For a natural number $N\in \mathbb{N},$ we denote by $[N] \defeq \{1,\dots, N\}.$ For a vector $x \in \RR^d$, the number $\|x\|$ denotes the Euclidean norm; for a matrix $A \in \RR^{d \times k},$ the number $\|A\|$ denotes the operator norm of $A$ with respect to the Euclidean norms on $\RR^d$ and $\RR^k.$ For $R > 0$, we let $B_R(\RR^d) \defeq \{x \in \RR^d\colon \|x\| \leq R\}$ denote the closed ball of radius $R$ in $\RR^d$; when the dimension $d$ is clear from context, we simply write $B_R(\RR^d) = B_R.$ We define $\log_+(R)\defeq\max(0,\log(R))$, where $\log$ is the natural logarithm. Given a real-valued function $f$, we denote by $\operatorname{Lip}(f)$ the (minimal) Lipschitz constant of $f$, with the convention that $\operatorname{Lip}(f) = \infty$ if $f$ is not Lipschitz continuous. Given a metric space $\mathcal{X}$, let $\mathcal{P}(\mathcal{X})$ be the set of Borel probability measures on $\mathcal{X}$.
Given $\mu$ and $\nu$ in $ \mathcal{P}(\mathcal{X})$ with finite first moment, we denote their $1$-Wasserstein distance $\sfW_1(\mu,\nu)$ by
\begin{align*}
	\sfW_1(\mu,\nu) &\defeq \inf\left\{\int_{\mathcal{X}\times\mathcal{X}} \|x-y\|\, \gamma(dx,dy)\colon \gamma \in \Pi(\mu,\nu) \right\} \\
	&= \sup \left\{\int_\mathcal{X} f\, d\mu - \int_\mathcal{X} f \,d\nu\colon \operatorname{Lip}(f) \leq 1 \right\},
\end{align*}
where $\Pi(\mu,\nu)$ denotes the set of couplings of $\mu$ and $\nu.$ Given $\mu\in\cP(\mathcal{X})$ on a set $\mathcal{X}$ and a function $T$ on $\mathcal{X},$ we denote by $T_{\#}\mu$ the pushforward measure. Given a measurable space $\mathcal{X}$ and a measurable subset $\mathcal{A} \subseteq \mathcal{X},$ we denote by $\mathbf{1}_{\mathcal{A}}\colon \mathcal{X} \rightarrow \{0,1\}$ the indicator function of $\mathcal{A}.$ Given positive functions $f$ and $g$, we use the notation $f(x) = O(g(x))$ or $f(x) \lesssim g(x)$ if there exists a positive constant $C$ such that $f(x) \leq C g(x)$ for all $x$ in a set of interest that should be clear from the context. We also write $f(x) = \Omega(g(x))$ if $g(x)=O(f(x))$. We often use $C$ to denote a positive constant whose value changes from line to line.

\section{Background}\label{sec-background}
This section provides the mathematical and probabilistic preliminaries on measure-theoretic transformers.
\subsection{Mean-Field Transformers}\label{subsec-meanfield-tf}
We begin with a review of the traditional self-attention mechanism. Let $X = (x_1, \ldots, x_N) \subset \RR^d$ denote an input sequence. A single-head self-attention module takes $X$ as an input and returns an output sequence $\Att_{\theta}(X) \subset \RR^k$ defined by
\begin{align*}
	\Att_{\theta}(X)_i = \frac{\sum_{j=1}^{N} \exp \bigl( \langle x_i, Ax_j \rangle \bigr) Vx_j}{\sum_{k=1}^{N} \exp \bigl( \langle x_i, Ax_k \rangle \bigr)} \qf i\in [N].
\end{align*}
Here, $V \in \RR^{k \times d}$ and $A \in \RR^{d \times d}$ denote the learnable parameters, and $\theta = (A,V)$. In short, an attention head updates tokens by computing a weighted average of $\{Vx_1, \dots, Vx_N\}$, where the weights are determined by the attention scores $\{\exp(\langle x_i, Ax_j \rangle)\}_{(i,j) \in [N]^2}.$ By defining the in-context map 
\begin{align}\label{eq: contextualmap}
	\Att_{\theta}(X,x) = \frac{\sum_{j=1}^{N} \exp \bigl( \langle x, Ax_j \rangle \bigr) Vx_j}{\sum_{k=1}^{N} \exp \bigl( \langle x, Ax_k \rangle \bigr)}
\end{align}
associated with the attention layer, we can write the $i$-th token of the output sequence as $\Att_{\theta}(X)_i = \Att_{\theta}(X,x_i).$ The function in \eqref{eq: contextualmap} is called an in-context map because it depends on the input sequence (or context) as well as $x.$ Notice that the in-context map depends on the sequence $X$ only through the empirical measure $\mu_N = \frac{1}{N} \sum_{i=1}^{N} \delta_{x_i}.$ This motivates us to define a mean-field self-attention head as a mapping $\Att_{\theta}\colon K \times \RR^d \rightarrow \RR^k$ with $K \subset \mathcal{P}(\RR^d)$ given by
\begin{align}\label{eq: mfselfattn}
	\Att_{\theta}(\mu,x) = \frac{\int \exp \bigl( \langle x, Ay \rangle \bigr) Vy\, \mu(dy)}{\int \exp \bigl( \langle x, Az \rangle \bigr) \, \mu(dz)}.
\end{align}
This formulation is valid for all probability measures for which the above integrals converge and recovers the traditional attention mechanism when $\mu$ is an empirical measure. Hereafter, for $K \subset \mathcal{P}(\RR^d),$ we refer to any mapping $K \times \RR^d \rightarrow \RR^k$ as an in-context map. A multi-head mean-field self-attention layer is an in-context map $\Gamma_{\theta}\colon K \times \RR^d \rightarrow \RR^d$ of the form
\begin{align}\label{eq: multiheadselfattn}
	\Gamma_{\theta}(\mu,x) = x + \sum_{h=1}^{H} W_h \Att_{\theta_h}(\mu,x),
\end{align}
where $W_1, \dots, W_H \in \RR^{d \times k}$, each $\Att_{\theta_h}$ is given by \eqref{eq: mfselfattn}, and the learnable parameters are $\theta = \{(\theta_h,W_h))\}_{h=1}^{H}.$ Equation~\eqref{eq: multiheadselfattn} is a direct generalization of the standard multi-head self-attention module for sequences to the measure-theoretic setting. Throughout the rest of this work, $\theta$ will be fixed or clear from context; we may simply abbreviate $\Att_{\theta}$ to $\Att$ and $\Gamma_{\theta}$ to $\Gamma.$

Given two in-context mappings $G_1\colon K_1 \times \RR^{d_1} \rightarrow \RR^{d_2}$ and $G_2\colon K_2 \times \RR^{d_2} \rightarrow \RR^{d_3}$ with $K_1 \subset \mathcal{P}(\RR^{d_1})$ such that $G_1(\mu,\slot)_{\#}\mu \in K_2$ for all $\mu \in K_1,$ there is a natural operation to produce a third in-context mapping $(G_2 \diamond G_1)\colon K_1 \times \RR^{d_1} \rightarrow \RR^{d_3}$ defined by
\begin{align*}
	(G_2 \diamond G_1)(\mu,x) = G_2\bigl(G_1(\mu,\slot)_{\#}\mu, G_1(\mu,x)\bigr).
\end{align*}
The operation `$\diamond$' is associative. Thus, it allows us to define compositions of in-context maps. A mean-field transformer is an in-context mapping defined by composing multi-head self-attention layers and feedforward layers. Recall that a (shallow) feedforward neural network is a function $\phi\colon \RR^d \rightarrow \RR^k$ of the form
\begin{align}\label{eq-shallow-nn}
	\phi(x) = W_2 \operatorname{ReLU}(W_1 x),
\end{align}
where $W_1 \in \RR^{m \times d}$, $W_2 \in \RR^{k \times m},$ and the activation function $\operatorname{ReLU}\colon \RR \rightarrow [0,\infty)$ defined by $\operatorname{ReLU}(x) = \max(0,x)$ is applied component-wise to the vector $W_1x.$
A \textbf{\textit{mean-field transformer}} is then defined as an in-context mapping of the form
\begin{align}\label{eqref: mftf}
	\mathsf{T} = \left(\phi_L \diamond \Gamma_{\theta_L} \diamond \dots \diamond \phi_1 \diamond \Gamma_{\theta_1} \right) \qf L \in \mathbb{N},
\end{align}
where each $\Gamma_{\theta_i}$ has the form \eqref{eq: multiheadselfattn} and each $\phi_i$ is a ReLU network of the form \eqref{eq-shallow-nn}, which we identify with the (constant in $ \mu $) in-context map $(\mu,x) \mapsto \phi_i(x)$. Hereafter, we use `$\mathsf{T}$' to denote a general mean-field transformer of the form \eqref{eqref: mftf}, whereas `$\Gamma$' refers to a single multi-head self-attention layer. \Citet{furuya2024transformers} proves that the set of mean-field transformers is dense in the set of continuous in-context mappings from $\mathcal{P}(\Omega) \times \Omega$ to $\RR^k$, where $\Omega \subset \RR^d$ is compact.

\begin{remark}
	In \eqref{eqref: mftf}, we could also allow the ReLU networks $\phi_1, \dots, \phi_L$ to be arbitrarily deep. However, since multi-head attention layers can represent the identity map, there is no generality lost in restricting $\phi_1, \dots, \phi_L$ to be shallow.
\end{remark}

\begin{remark}
	Our sub-Gaussian analysis omits normalization layers \citep{ba2016layer}. The main obstruction presented by layer norm (LN) is that it breaks the continuity of transformers. Developing analogs of our results which are compatible with LN is an interesting problem for future work.
    Nonetheless, although pre-LN transformers only operate on compactly supported measures \citep{kan2025stability}, our work is relevant for understanding post-LN transformers that do accommodate unbounded support in the initial input layer \citep{hargenrader2026attention}. While often harder to train, post-LN models can exhibit improved performance in some settings \citep{liu2020understanding}.
\end{remark}

A mean-field transformer $\mathsf{T}\colon K \times \RR^d \rightarrow \RR^k$ naturally defines an operator from $K$ to $\mathcal{P}(\RR^k)$ by
\begin{align*}
	\mu \mapsto \mathsf{T}(\mu,\slot)_{\#}\mu.
\end{align*}
We refer to this operator as a \textbf{\textit{mean-field transformer operator}}. Notice that if $\mathsf{T}_1$ and $\mathsf{T}_2$ are mean-field transformers (or more general in-context mappings) for which $(\mathsf{T}_2 \diamond \mathsf{T}_1)$ is well-defined, then
\begin{align*}
	(\mathsf{T}_2 \diamond \mathsf{T}_1)(\mu,\slot)_{\#}\mu = \mathsf{T}_2(\mathsf{T}_1(\mu,\slot)_{\#}\mu, \slot)_{\#} (\mathsf{T}_1(\mu,\slot)_{\#}\mu).
\end{align*}
In other words, given two mean-field transformers, the in-context composition rule corresponds to the standard composition rule for their corresponding measure-to-measure operators.

\subsection{Sub-Gaussian Probability Distributions}\label{subsec-background-subG}
Sub-Gaussian probability distributions are central objects in high-dimensional probability and machine learning which generalize the tail decay properties of the Gaussian distribution. Given $\alpha \geq 1$ and $\beta > 0$, we define the class of $(\alpha,\beta)$-sub-Gaussian measures on $\RR^d$, denoted by $\sg_{\alpha,\beta}(\RR^d),$ to consist of all measures $\mu \in \mathcal{P}(\RR^d)$ with square-exponential tail decay:
$$ \sup_{R > 0} \mu\big(\{x\in\RR^d\colon \|x\| > R\} \big)e^{\frac{R^2}{\beta}} \leq \alpha.
$$
The sets $\{\sg_{\alpha,\beta}(\RR^d)\}_{\alpha \geq 1, \beta > 0}$ are monotonic, i.e., if $\alpha_1 \leq \alpha_2$ and $\beta_1 \leq \beta_2$, then $\sg_{\alpha_1,\beta_1}(\RR^d) \subseteq \sg_{\alpha_2,\beta_2}(\RR^d).$ We then define
\begin{align*}
	\sg(\RR^d) \defeq \bigcup_{\alpha \geq 1, \beta > 0} \sg_{\alpha,\beta}(\RR^d).
\end{align*} This definition coincides with the classical definition of sub-Gaussianity, i.e., $\mu \in \sg(\RR^d)$ if and only if $X \sim \mu$ has finite sub-Gaussian norm \citep[Chapter~2]{vershynin2018high}. If $\mu \in \sg_{\alpha,\beta}(\RR^d),$ then we refer to $\alpha$ and $\beta$ as the sub-Gaussian parameters of $\mu.$ 

\section{Main Results}\label{sec-results}
This section develops contributions (C1)--(C5).
\begin{remark}\label{rmk-w1}
    In this paper, most of our upper bounds on the $\sfW_1$ metric are based on coupling arguments. Thus, we expect that our results also hold for $\sfW_p$ with $p > 1,$ up to constant factors which depend on $p.$ We leave this extension to future work.
\end{remark}

\subsection{Transformers Preserve Sub-Gaussianity}\label{subsec-results-subG} As alluded to previously, defining mean-field transformers on probability measures with unbounded support requires care. Indeed, if $\Att_{\theta}$ is a self-attention layer with parameters $A$ and $V$, then $\Att_{\theta}(\mu,x)$ is well-defined if and only if the integrals
\begin{align*}
	\int e^{\langle x, Ay \rangle} Vy \, \mu(dy) \qa \int e^{\langle x, Ay \rangle} \, \mu(dy)
\end{align*}
converge. If these integrals converge for every $x \in \RR^d$, then $\mu$ necessarily has sub-exponential tails. The matter is further complicated by considering compositions of attention layers. Indeed, it could be the case that $\mu$ is sufficiently light-tailed to make $\Att_{\theta}(\mu,\slot)$ well-defined, but $\Att_{\theta}(\mu,\slot)_{\#}\mu$ is heavy-tailed. To circumvent these concerns, we would like to identify a set $K \subset \mathcal{P}(\RR^d)$ on which transformers are well-defined and \textit{stable}, in the sense that $\Psi(K) \subset K$ for any mean-field transformer operator $\Psi$. Below, we establish $\sg(\RR^d)$ as having this desired property. The key insight is that if $\mu \in \sg(\RR^d)$, then $\Att_{\theta}(\mu,\slot)$ grows at most linearly.

\begin{proposition}\label{prop-subg-stability}
	Fix $\alpha \geq 1$ and $\beta > 0.$ Let $\Gamma$ be a mean-field self-attention operator given by
	\begin{align}\label{eqn:mf_sa_operator}
		\Gamma(\mu,x) = x + \sum_{h=1}^{H} W_h \frac{\int e^{\langle x, A_hy \rangle} V_hy \, \mu(dy)}{\int e^{\langle x, A_hy \rangle} \, \mu(dy)}. 
	\end{align}
	Then there exist constants $K_1, K_2 > 0$, depending only on $\alpha$, $\beta$, $H$, and the attention weights, such that
	\begin{align*}
		\sup_{\mu \in \sg_{\alpha,\beta}(\RR^d)}\|\Gamma(\mu,x)\| \leq K_1 + K_2\|x\|\qfa x \in \RR^d.
	\end{align*}
\end{proposition}
\Citet{castin2025unified} established that when $\mu$ is a Gaussian measure, the mapping $x \mapsto \Att_{\theta}(\mu,x)$ is linear. As a consequence, $\mu \mapsto \Att_{\theta}(\mu,\slot)_{\#}\mu$ sends Gaussian measures to Gaussian measures. In contrast, our result asserts that when $\mu$ is \textit{sub-Gaussian}, the self-attention map $\Att_{\theta}(\mu,\slot)$ grows \textit{affine-linearly}; thus, its pushforward preserves sub-Gaussianity.

A consequence of Proposition~\ref{prop-subg-stability} is that mean-field transformers map $\sg(\RR^d)$ to $\sg(\RR^k).$

\begin{theorem}\label{thm-subG-stable}
	Let $\mathsf{T}$ be a mean-field transformer. Then $\mathsf{T}(\mu,\slot)_{\#}\mu \in \sg(\RR^k)$ for every $\mu \in \sg(\RR^d).$
\end{theorem}
\begin{proof}
	Let $\mu \in \sg(\RR^d).$ By Proposition \ref{prop-subg-stability} and the fact that ReLU feedforward networks are Lipschitz continuous, there exist constants $K_1, K_2 > 0$ (depending on the parameters of $\mathsf{T}$ and the sub-Gaussian parameters of $\mu$) such that $\|\mathsf{T}(\mu,x)\| \leq K_1 + K_2\|x\|$ for all $x \in \RR^d.$ The conclusion follows by Lemma~\ref{lem-subG-pushforward}.
\end{proof}
Theorem~\ref{thm-subG-stable} identifies sub-Gaussianity as a natural tail assumption for mean-field transformers. Importantly, it applies not only to individual attention layers, but also to arbitrarily deep transformers.

\subsection{H\"older Regularity of Transformers on Sub-Gaussian Domains}\label{subsec-results-holder}
Understanding the regularity of neural architectures is important, for instance, to ensure their robustness under corruptions of the data, to characterize the functions which they can efficiently approximate, and to prove generalization bounds. In this subsection, we study the regularity of measure-to-measure operators of the form $\mu \mapsto \mathsf{T}(\mu,\slot)_{\#}\mu,$ where $\mathsf{T}$ is a mean-field transformer.  With the results of the previous subsection in mind, we restrict the domain of these operators to be $\sg_{\alpha,\beta}(\RR^d)$ for a fixed (but arbitrary) pair $(\alpha,\beta) \in [1,\infty) \times (0,\infty).$ Previous works have studied the regularity of such mappings on the domain $\mathcal{P}(B_R)$, where it was shown that mean-field self-attention operators are Lipschitz continuous \citep{castin2024smooth}. In particular, these results established lower bounds on the Lipschitz constant of a mean-field self-attention map $\Att_{\theta}$ of the form
\begin{align*}
	\operatorname{Lip}_{\mathcal{P}(B_R)}(\Att_{\theta}) = \Omega\left( e^{\Omega(R^2)}\right). 
\end{align*}
These lower bounds suggest that mean-field transformer operators might not be Lipschitz continuous on $\sg_{\alpha,\beta}(\RR^d),$ since $\sg_{\alpha,\beta}(\RR^d)$ contains measures with unbounded support. Instead, we prove that mean-field transformer operators are H\"older continuous on $\sg_{\alpha,\beta}(\RR^d)$. We first state a result for a multi-head self-attention layer, where we have a precise (but possibly sub-optimal) estimate of the H\"older exponent.

\begin{theorem}\label{thm-holder-attn}
	Let $\Gamma\colon \sg_{\alpha,\beta}(\RR^d) \times \RR^d \rightarrow \RR^d$
	be a mean-field self-attention operator as in \eqref{eqn:mf_sa_operator}.
	Define $c_A \defeq \max_{h \in [H]} \|A_h\|,$ $c_V \defeq \max_{h \in [H]} \|V_h\|,$ and $c_W \defeq \max_{h \in [H]} \|W_h\|.$ Further, define
	\begin{align*}
		&k_1 \defeq 5c_A, \; k_2 \defeq \min \left(\frac{1}{\beta}, \frac{1}{2(\beta^3 \alpha^2 c_A^2+1)}\right), \\
        &\qquad\qquad \gamma \defeq \frac{k_2}{k_1+k_2}.
	\end{align*}
	Then there exists a constant $C$, depending on $\alpha,\beta, H, c_A,$ $c_V$, and $c_W,$ such that, for all $\epsilon > 0,$
	\begin{align*}
        \sfW_1\bigl(\Gamma(\mu,\slot)_{\#}\mu,\Gamma(\mu_0,\slot)_{\#}\mu_0\bigr) \leq C \log_+(\epsilon^{-1}) \epsilon^{\gamma}
	\end{align*}
    holds for all $\mu$ and $\mu_0$ in $\sg_{\alpha,\beta}(\RR^d)$ with $\sfW_1(\mu,\mu_0) \leq \epsilon.$
	In particular, by enlarging the constant and decreasing the H\"older exponent too, it follows that there exists $C(\alpha,\beta,H,c_A,c_V) > 0$ and $\gamma' \in (0,\gamma)$ such that
	\begin{align*}
		\sfW_1\bigl(\Gamma(\mu,\slot)_{\#}\mu, \Gamma(\mu_0,\slot)_{\#}\mu_0\bigr) \leq C\,  \sfW_1(\mu,\mu_0)^{\gamma'}
	\end{align*}
	for all $\mu\in \sg_{\alpha,\beta}(\RR^d)$ and all $\mu_0 \in \sg_{\alpha,\beta}(\RR^d).$
\end{theorem}
Theorem~\ref{thm-holder-attn} is proven by truncating the supports of $\mu$ and $\mu_0$, controlling the $\sfW_1$ error between truncated versions of $\mu$ and $\mu_0$, and trading off the truncation error with the Lipschitz constant. The details can be found in Appendix~\ref{subsec-proofs-holder}. The dependence of the H\"older exponent on $\|A\|$ is unavoidable from our proof strategy because the Lipschitz constant of $\mu \mapsto \Gamma(\mu,\slot)_{\#}\mu$ on $\mathcal{P}(B_R)$ is $\Omega(e^{\Omega(c_A R^2)}),$ while for $\mu \in \sg_{\alpha,\beta}(\RR^d),$ the Wasserstein distance between $\mu$ and its truncation to $\mathcal{P}(B_R)$ is $O(e^{-R^2/\beta})$. However, we do not attempt to optimize the exponent $\gamma$ in our proof. The final estimate crudely removes the log factor by enlarging the H\"older constant and reducing the H\"older exponent, thus establishing a \textit{bona fide} H\"older continuity estimate of the mean-field multi-head self-attention layer.

We extend the H\"older regularity to general mean-field transformers, though the generality of the statement requires sacrificing control on the H\"older exponent.

\begin{theorem}\label{thm-holder-tf}
	Fix $d \in \mathbb{N}$ with $d \geq 3.$ Let $\mathsf{T}\colon \sg_{\alpha,\beta}(\RR^d) \times \RR^d \rightarrow \RR^k$ be a mean-field transformer. Then there exist constants $C > 0$ and $\gamma \in (0,1)$ (depending on $\mathsf{T}$, $\alpha,$ and $\beta$) such that
	\begin{align*}
		\sfW_1\bigl(\mathsf{T}(\mu,\slot)_{\#}\mu,\mathsf{T}(\mu_0,\slot)_{\#}\mu_0\bigr) \leq C \epsilon^{\gamma}
	\end{align*} 
    holds for all $\mu$ and $\mu_0$ in $ \sg_{\alpha,\beta}(\RR^d)$ with $\sfW_1(\mu,\mu_0) \leq \epsilon.$
\end{theorem}
The proof can be found in Appendix \ref{subsec-proofs-holder}. Theorem \ref{thm-holder-tf} complements the results of Section \ref{subsec-results-subG} by characterizing the smoothness of mean-field transformers on sub-Gaussian data. A natural application of these regularity results is to generalization error estimates for learning---e.g., via empirical risk minimization---on the space of probability measures with mean-field transformers. We leave this direction to future work.

\begin{remark}
    The assumption that $d \geq 3$ in Theorem \ref{thm-sample-tf} is only made to simplify the presentation and avoid a more complex case-based result. In more detail, the proof of Theorem \ref{thm-sample-tf} uses a Wasserstein law of large numbers of the form $\E \sfW_1(\mu,\mu_N) \lesssim N^{-1/d},$ for which the cases $d=1$ and $d=2$ are exceptional. Our bound directly generalizes to these cases, but we choose to present the result in its simplest form.
\end{remark}

\subsection{Sample Complexity of Transformers}\label{subsec-results-samples}
While the mean-field framework is a convenient setting to mathematically study transformers, probability measures must be quantized in practice. A natural means of quantization is independent and identically distributed (iid) sampling. As an application of our regularity results, we estimate the sample complexity of mean-field transformer operators in the iid setting. Given $\mu \in \mathcal{P}(\RR^d),$ we use $\mu_N$ to denote the (random) empirical measure associated with $\mu$. That is, if $x_n\sim \mu$ for $n\in[N]$ are iid, then $\mu_N=\frac{1}{N}\sum_{n=1}^N\delta_{x_n}$.

\begin{theorem}\label{thm-sample-tf}
	Fix $d \in \mathbb{N}$ with $d \geq 3,$ $\alpha \geq 1$, and $\beta > 0.$ Let $\mathsf{T}$ be a mean-field transformer. Let $\gamma \in (0,1)$ be a H\"older exponent of $\mathsf{T}$ on $\sg_{\alpha,\beta}(\RR^d)$. Then there exist constants $C>0$ and $C' \geq 1$, depending on $\mathsf{T}$, $\alpha,$ $\beta$, $\gamma$, and $d$, such that
	\begin{align*}
		&\sup_{\mu \in \sg_{\alpha,\beta}(\RR^d)} \E \, \sfW_1\bigl(\mathsf{T}(\mu,\slot)_{\#}\mu,\mathsf{T}(\mu_N,\slot)_{\#}\mu_N\bigr) \\
        &\quad\quad\leq C \log(N)^{\gamma/2}\Bigl(1+\log\bigl(1+N^{1/d}\bigr)\Bigr)^{\gamma/2}N^{-\gamma/d}
	\end{align*}
    for all $N \geq C'.$
\end{theorem}
Theorem~\ref{thm-sample-tf} does not follow simply by applying the H\"older regularity of $\mathsf{T}.$ The issue is that the empirical measure $\mu_N$ may not belong to $\sg_{\alpha,\beta}(\RR^d).$ To deal with this, we first prove that when $N$ is sufficiently large, there exist constants $\widehat{\alpha} \geq 1$ and $\widehat{\beta} > 0$, which depend on $\alpha$ and $\beta$ but are independent of $N$, such that $\mu_N \in \sg_{\widehat{\alpha},\widehat{\beta}}(\RR^d)$ with high probability; see Proposition~\ref{prop-subG-empirical} for a precise statement. We denote the event on which this inclusion holds by $\mathcal{A}_N$. On $\mathcal{A}_N,$ we use the H\"older regularity of $\sfT$ and a uniform Wasserstein law of large numbers to control the averaged difference between $\mathsf{T}(\mu,\slot)_{\#}\mu$ and $\mathsf{T}(\mu_N,\slot)_{\#}\mu_N$. We conclude by proving that the residual error term incurred by the tail probability of the bad set $\mathcal{A}_N^c$ is of higher order. The details can be found in Appendix~\ref{subsec-proofs-holder}.

To complement Theorem~\ref{thm-sample-tf}, we provide an analogous result under the assumption that $\mu$ has compact support. This assumption leads to an improved convergence rate with respect to the sample size $N$ and highlights a difference between the behavior of transformers on bounded and unbounded data. Although the next result is likely not novel, we were unable to locate a precise reference for it. Thus, we decided to present it in case it is of use in other contexts. 

\begin{proposition}
	\label{thm-sample-tf-compact}
	Fix $d \in \mathbb{N}$ with $d \geq 3.$ Let $R > 0$ and $N \in \mathbb{N},$ and let $\mathsf{T}\colon \mathcal{P}(\RR^d) \times \RR^d \rightarrow \RR^k$ be a mean-field transformer. Then there is a constant $C>0$, depending on $d$, $R$, and $\mathsf{T}$, such that
	\begin{align*}
		&\sup_{\mu \in \mathcal{P}(B_R(\RR^d))} \E\, \sfW_1\bigl(\mathsf{T}(\mu,\slot)_{\#}\mu,\mathsf{T}(\mu_N,\slot)_{\#}\mu_N\bigr) \\ 
        &\quad\quad \leq C \left(1+\sqrt{\log(1+N^{1/d})} \right)N^{-1/d}.
	\end{align*}
\end{proposition}
The proof of Proposition~\ref{thm-sample-tf-compact}, which is presented in Appendix~\ref{subsec-proofs-samples}, is substantially simpler than the proof of Theorem~\ref{thm-sample-tf} for two reasons: first, the mapping $\mu \mapsto \mathsf{T}(\mu,\slot)_{\#}\mu$ is globally Lipschitz on $\mathcal{P}(B_R),$ and second, the empirical measure $\mu_N$ belongs to $\mathcal{P}(B_R)$ almost surely.

\subsection{Transformers With Mean-Field Cross-Attention}\label{subsec-results-cross}
Given two input sequences $X = (x_1, \dots, x_N) \subset \RR^d$ and $Z = (z_1, \dots, z_M) \subset \RR^{d'},$ a (single-head) cross-attention layer maps $X$ and $Z$ to a third sequence $Y = (y_1, \dots, y_M) \subset \RR^{d'}$, where
\begin{align*}
	y_i = z_i + \frac{\sum_{n=1}^{N} e^{\langle z_i, Ax_n \rangle} Vx_n}{\sum_{j=1}^{N} e^{\langle z_i, Ax_j \rangle}}
\end{align*}
for $i\in[M]$ and $V$ and $A$ belong to $\RR^{d' \times d}.$ Taking the limit as the number of tokens $N$ and $M$ tend to infinity, we consider a mean-field analog of transformers with the cross-attention architecture by considering operators
\begin{align*}
	(\mu,\eta) \mapsto \Gamma(\mu,\slot)_{\#} \eta,
\end{align*}
where $\Gamma$ is the in-context mapping associated with a mean-field attention layer of the form \eqref{eq: multiheadselfattn}. To complement the results of the previous subsections, we derive regularity and sample complexity estimates for mean-field transformers with cross-attention. We begin by establishing a joint H\"older continuity bound when the mapping above is defined by a multi-head self-attention layer.

\begin{theorem}\label{thm: holder-cross-attn}
	Let $\Gamma\colon \sg_{\alpha,\beta}(\RR^d) \times \RR^{d'} \rightarrow \RR^{d'}$ be a mean-field attention operator given by
	\begin{align}\label{eqn:gamma_map_cross}
		\Gamma(\mu,x) \defeq x + \sum_{h=1}^{H} W_h \frac{\int e^{\langle x, A_hy \rangle} V_hy\, \, \mu(dy)}{\int e^{\langle x, A_hz \rangle}\, \, \mu(dz)}.
	\end{align}
	Let $c_A,$ $c_V$, and $c_W$ be as in Theorem~\ref{thm-holder-attn}. Define
	\begin{align*}
        k_1 &\defeq \frac{1}{2\beta(\beta^3 \alpha^2 c_A^2 + 1)}, \; k_2 \defeq 2 \sqrt{\frac{2c_A^2}{\beta^3 \alpha^2 c_A^2 + 1}}, \\ &\qquad\qquad \gamma \defeq \frac{k_2}{k_1+k_2}.
	\end{align*}
	Then there exists a constant $C=C(\alpha,\beta, H,c_A, c_V) > 0$ such that for all $\mu$ and $\mu_0$ in $\sg_{\alpha,\beta}(\RR^d)$ and all $\eta$ and $\eta_0$ in $\sg_{\alpha,\beta}(\mathbb{R}^{d'})$,
	it holds that
	\begin{align*}
		&\sfW_1\bigl(\Gamma(\mu,\slot)_{\#}\eta, \Gamma(\mu_0,\slot)_{\#}\eta_0\bigr) \leq C \times\\
        &\Bigl(1 \vee \log_+\bigl(\sfW_1(\mu,\mu_0)^{-1}\bigr)^{\frac{3}{2}}\Bigr) \Bigl( \sfW_1(\mu,\mu_0)^{\gamma} + \sfW_1(\eta,\eta_0) \Bigr).
	\end{align*}
\end{theorem}
Theorem~\ref{thm: holder-cross-attn} establishes joint H\"older regularity of the mean-field cross-attention mapping, up to logarithmic factors. Note that the estimate is not symmetric in $\mu$ and $\eta$ because the mapping $(\mu,\eta) \mapsto \Gamma(\mu,\slot)_{\#}\eta$ is not symmetric. Rather, the estimate shows that the measure-to-measure operator associated with cross-attention is smoother as a function of $\eta$ than as a function of $\mu$; in particular, cross-attention is nearly \textit{Lipschitz continuous} with respect to the second argument, up to a logarithmic factor. This asymmetry is due to the difference in scale between the local Lipschitz constant of the mapping $(\mu,x) \mapsto \Gamma(\mu,x)$ with respect to each variable. Theorem~\ref{thm: holder-cross-attn} can also be generalized to repeated compositions of cross- and self-attention layers and feedforward neural networks, but the propagation of log factors across layers is more complicated than for mean-field transformers with self-attention. 
Moreover, this deep setting is also more challenging because the implied deep transformer in-context map can depend on both $\mu$ and $\eta$, instead of just $\mu$.
For these reasons, we choose not to generalize Theorem~\ref{thm: holder-cross-attn} beyond a single attention layer.

Paralleling Theorem~\ref{thm-sample-tf}, we also establish a sample complexity bound for transformers with cross-attention. To simplify the presentation, we again restrict to a single attention layer, but the result can straightforwardly be generalized to deep transformers at the expense of a looser scaling law. In the result below, we consider two assumptions on the pair of measures $(\mu,\eta)$:
when $\mu \in \mathcal{P}(B_R)$ and $\eta \in \sg_{\alpha,\beta}(\RR^d)$ for some $R > 0,$ $\alpha \geq 1,$ and $\beta > 0,$ we obtain nonparametric sample complexity rates analogous to those in Theorem~\ref{thm-sample-tf}; on the other hand, when both $\mu$ and $\eta$ belong to $\mathcal{P}(B_R),$ we find that the sample complexity with respect to $\mu$ achieves a dimension-free rate of $N^{-1/2}$, up to log factors. We leave extending the cross-attention sample complexity bounds in the case where both $\mu$ and $\gamma$ have unbounded support to future work.

\begin{theorem}\label{thm-sample-cross-attn}
	Fix $d \in \mathbb{N}$ with $d \geq 3.$ Fix $\alpha \geq 1,$ $\beta > 0,$ and $R > 0.$ Let $\Gamma\colon \mathcal{P}(\RR^d) \times \RR^{d'} \rightarrow \RR^{d'}$ be a mean-field multi-head attention layer as in \eqref{eqn:gamma_map_cross}.
	Let $c_A$, $c_V$, and $c_W$ be as defined in Theorem~\ref{thm-holder-attn}. Then there exist constants $C_1 > 0$, $C_2 > 0,$ and $C_3 \geq 1,$ which depend on $d,d'$, $\alpha,$ $\beta$, $R$, $H$, $c_A$, and $c_V,$ such that for all $M\in\N$ and all $N \geq C_3,$ it holds that
	\begin{align*}
		&\sup_{\mu \in \mathcal{P}(B_R), \eta \in \sg_{\alpha,\beta}(\RR^{d'})} \E\, \sfW_1\bigl(\Gamma(\mu,\slot)_{\#}\eta, \Gamma(\mu_N,\slot)_{\#}\eta_M\bigr) \\ &\ \ \leq C_1 \Bigl( \mathsf{polylog}(N) N^{-\frac{1}{d\left(4c_A \beta  +1\right)}} + \mathsf{polylog}(M)M^{-1/d} \Bigr)
	\end{align*}
	and
	\begin{align*}
		&\sup_{\mu \in \mathcal{P}(B_R),\eta \in \mathcal{P}(B_R)} \E\, \sfW_1\bigl(\Gamma(\mu,\slot)_{\#}\eta, \Gamma(\mu_N,\slot)_{\#}\eta_M\bigr) \\ &\quad \leq C_2 \left(N^{-1/2} + \left(1+\sqrt{\log\bigl(1+M^{1/d}\bigr)} \right) M^{-1/d} \right).
	\end{align*}
\end{theorem}

As is the case in Theorem \ref{thm: holder-cross-attn}, Theorem \ref{thm-sample-cross-attn} suggests that the sample complexity of cross-attention differs with respect to each variable. These scaling laws may provide practical suggestions for how long the relative contexts of each component should be. This is particularly relevant for generative modeling settings, where $\eta$ may denote a reference measure from which new samples can be readily generated.

\subsection{Consequences for Measure-to-Measure Approximation With Transformers}\label{sec:main_approx}
We conclude the section with an application of our results to the approximation of measure-to-measure operators by transformers. Recently, \citet{furuya2026universal} established sufficient conditions on a $\sfW_p$-compact subset $K \subset \mathcal{P}(\RR^d)$ under which the pushforward operators induced by transformers can approximate continuous measure-to-measure operators to arbitrary accuracy, uniformly over $K.$ However, this result does not have immediate consequences for transformers evaluated at empirical measures because the hypotheses of the result prohibit $K$ from containing atomic measures. Below, we state a generalization of the result to transformers evaluated at empirical measures, where the sampling error is instead controlled using Theorem~\ref{thm-sample-tf}.

\begin{theorem}\label{thm-approx-sample-tf}
    Fix $\alpha \geq 1$ and $\beta > 0.$ Let $\mathcal{P}_1(\R^d)$ denote the subset of probability measures in $\mathcal{P}(\R^d)$ with finite first moment. Let $K \subset \sg_{\alpha,\beta}(\RR^d)$, and suppose there exists a non-atomic measure $\tau \in \mathcal{P}(\RR^d)$ such that $\mu$ is absolutely continuous with respect to $\tau$ for every $\mu \in K.$ 
    Let $F\colon (K, \sfW_1) \rightarrow (\mathcal{P}_1(\RR^d), \sfW_1)$ be continuous.
    Then for any $\epsilon > 0,$ there exists a transformer $\mathsf{T}_{\epsilon}$ of the form \eqref{eqref: mftf} and constants $C_{\epsilon} > 0$ and $\gamma_{\epsilon} \in (0,1)$ (depending on $d,$ $\alpha,$ $\beta$, and $T_{\epsilon}$) such that
    \begin{align*}
        &\sup_{\mu \in K} \E\, \sfW_1\bigl(\mathsf{T}_{\epsilon}(\mu_N,\slot)_{\#}\mu_N, F(\mu)\bigr) < \frac{\epsilon}{2} \ +\\
        &\quad\quad  C_{\epsilon} \log(N)^{\gamma_{\epsilon}/2}\left(1+\log\bigl(1+N^{1/d}\bigr)\right)^{\gamma_{\epsilon}/2}N^{-\gamma_{\epsilon}/d}
    \end{align*}
    for $N$ large enough. In particular, for each $\epsilon > 0$, there exists $N_{\epsilon} \in \mathbb{N}$ (depending on $d,$ $\alpha$, $\beta$, and $\epsilon$) such that
    \begin{align*}
         &\sup_{\mu \in K} \E\, \sfW_1\bigl(\mathsf{T}_{\epsilon}(\mu_N,\slot)_{\#}\mu_N, F(\mu)\bigr) < \epsilon
    \end{align*}
    for all $N \geq N_{\epsilon}$.
\end{theorem}
Theorem \ref{thm-approx-sample-tf} states that mean-field transformer operators are universal approximators for continuous measure-to-measure operators, \textit{even when the input measure is only accessible by samples}, provided that the number of samples is sufficiently large and the compact set $K$ upon which the uniform approximation occurs is uniformly subgaussian and absolutely continuous with respect to a common non-atomic reference. Moreover, for a fixed approximating transformer $\mathsf{T}_{\epsilon}$, the sampling error converges to zero at a polynomial rate, up to log factors. The assumption that $K$ comprises measures which are dominated by a common non-atomic reference measure ensures that continuous operators on $K$ can be universally approximated by pushforwards, even when $\mu\mapsto F(\mu)$ does not have an exact pushforward representation; see Assumption~3.1 in the work of \citet{furuya2026universal} for a more general sufficient condition. Indeed, without any atomloss assumptions on the measures belonging to $K,$ it is easy to construct examples of continuous operators on $K$ which cannot be approximated by pushforwards.

\section{Conclusion}
This work establishes mathematical properties of the measure-to-measure operators defined by transformers. The main contributions are a stability result stating that transformers map sub-Gaussian measures to sub-Gaussian measures, a H\"older continuity bound in the $1$-Wasserstein distance, and an estimate on the error propagation between a true probability measure and its empirical counterpart along a transformer. Analogs of these results are provided for transformers which use the cross-attention mechanism. The results herein motivate several directions for future work. First, our sample complexity results suffer from a curse of dimensionality which is intrinsic to high-dimensional nonparametric statistical problems; it is unclear whether the parametric complexity of transformers can be leveraged to improve these rates, or whether faster rates occur for probability measures with low-dimensional structure. Second, motivated by language models, we would like to prove analogs of our sample complexity results for empirical measures of non-iid data, e.g., for tokens generated autoregressively. Last, our results on regularity and sample complexity motivate the development of a generalization error analysis of transformers for supervised learning problems on spaces of probability measures. We leave these questions to future work.

\clearpage
\subsubsection*{AI Use Statement}
In this work, we used generative AI tools for formulating mathematical claims and providing critical ingredients to prove mathematical claims. We did not use generative AI tools for generating synthetic data sets, developing theoretical models or conceptual frameworks, assisting in the writing of proofs, proposing or refining hypotheses, designing or providing feedback on research methodology or experiments, implementing methods, assisting with translation, cleaning and reformatting datasets, supporting qualitative and thematic data analysis, or interpreting results. We have reviewed all AI-assisted work. In more detail, the proof ingredients for Proposition~\ref{prop-subg-stability} and Lemma~\ref{lem: gradientbound} came about between the authors and ChatGPT 5.6. We take responsibility for the final content of this work, including text, claims, or artifacts produced with the aid of generative AI.

 \subsubsection*{Acknowledgments}
 The research of N.H.N. is supported by a Klarman Fellowship through Cornell University's College of Arts \& Sciences and by startup funds at The University of Texas at Austin. T.F. acknowledges funding from JSPS KAKENHI (grants JP24K16949 and 25H01453), JST CREST (JPMJCR24Q5), and JST ASPIRE (JPMJAP2329).

\bibliographystyle{apalike}

\renewcommand{\refname}{References}
\renewcommand{\bibname}{References}
\renewcommand{\bibsection}{\subsubsection*{References}}

\bibliography{ref}

\section*{Checklist}

\begin{enumerate}
  \item For all models and algorithms presented, check if you include:
  \begin{enumerate}
    \item A clear description of the mathematical setting, assumptions, algorithm, and/or model. \textbf{Yes.}
    \item An analysis of the properties and complexity (time, space, sample size) of any algorithm. \textbf{Not Applicable}
    \item (Optional) Anonymized source code, with specification of all dependencies, including external libraries. \textbf{Not Applicable}
  \end{enumerate}

  \item For any theoretical claim, check if you include:
  \begin{enumerate}
    \item Statements of the full set of assumptions of all theoretical results. \textbf{Yes}
    \item Complete proofs of all theoretical results. \textbf{Yes}
    \item Clear explanations of any assumptions. \textbf{Yes}  
  \end{enumerate}

  \item For all figures and tables that present empirical results, check if you include:
  \begin{enumerate}
    \item The code, data, and instructions needed to reproduce the main experimental results (either in the supplemental material or as a URL). \textbf{Not Applicable}
    \item All the training details (e.g., data splits, hyperparameters, how they were chosen). \textbf{Not Applicable}
    \item A clear definition of the specific measure or statistics and error bars (e.g., with respect to the random seed after running experiments multiple times). \textbf{Not Applicable}
    \item A description of the computing infrastructure used. (e.g., type of GPUs, internal cluster, or cloud provider). \textbf{Not Applicable}
  \end{enumerate}

  \item If you are using existing assets (e.g., code, data, models) or curating/releasing new assets, check if you include:
  \begin{enumerate}
    \item Citations of the creator If your work uses existing assets. \textbf{Not Applicable}
    \item The license information of the assets, if applicable. \textbf{Not Applicable}
    \item New assets either in the supplemental material or as a URL, if applicable. \textbf{Not Applicable}
    \item Information about consent from data providers/curators. \textbf{Not Applicable}
    \item Discussion of sensible content if applicable, e.g., personally identifiable information or offensive content. \textbf{Not Applicable}
  \end{enumerate}

  \item If you used crowdsourcing or conducted research with human subjects, check if you include:
  \begin{enumerate}
    \item The full text of instructions given to participants and screenshots. \textbf{Not Applicable}
    \item Descriptions of potential participant risks, with links to Institutional Review Board (IRB) approvals if applicable. \textbf{Not Applicable}
    \item The estimated hourly wage paid to participants and the total amount spent on participant compensation. \textbf{Not Applicable}
  \end{enumerate}
\end{enumerate}

\clearpage
\appendix
\thispagestyle{empty}

\onecolumn
\aistatstitle{Supplementary Materials for:\\
\mytitle}

\section{Proofs}\label{app:proofs}
This appendix collects the proofs of results presented in the main text.

\subsection{Proofs for Section~\ref{subsec-results-subG}:~\nameref*{subsec-results-subG}}\label{subsec-proofs-subG-results}
Before proving Proposition \ref{prop-subg-stability}, we need to prove a technical lemma on the growth of gradients of convex functions.

\begin{lemma}\label{lem: gradientbound}
	Let $f\colon \RR^d \rightarrow \RR$ be a differentiable convex function such that there exist constants $A,B > 0$ for which $|f(x)| \leq A + B\|x\|^2$ for all $\|x\| \geq 1.$ Then
	$$ \|\nabla f(x)\| \leq (2A+B) + 5B\|x\|
	$$
	for all $\|x\| > 1.$
\end{lemma}
\begin{proof}
	Fix $x$ with $\|x\| > 1.$ Let $h > 0$ be a positive number to be chosen precisely later, and let $u \in \RR^d$ be an arbitrary unit vector. We apply the first order condition for convex functions to obtain
	\begin{align*}
		f(x+hu) \geq f(x) + h \nabla f(x) \cdot u,
	\end{align*}
	or equivalently,
	\begin{align*}
		\nabla f(x) \cdot u &\leq \frac{f(x+hu) - f(x)}{h}.
	\end{align*}
    If $\|x+hu\| > 1,$ then it holds that
    \begin{align*}
        \frac{f(x+hu) - f(x)}{h} &\leq \frac{2A + B(\|x+hu\|^2 + \|x\|^2)}{h} \\
		&= \frac{2A + 2B \|x\|^2}{h} + 2B x \cdot u + Bh \\
		&\leq \frac{2A + 2B \|x\|^2}{h} + 2B \|x\| + Bh.
    \end{align*}
    If $\|x+hu\| < 1,$ we can write $x+hu = \lambda z_1 + (1-\lambda)z_2$ for some $\lambda \in (0,1)$ and $z_1, z_2 \in \mathbb{S}^{d-1}.$ Convexity of $f$ implies that
    \begin{align*}
        |f(x+hu)| \leq \max(f(z_1),f(z_2)) \leq A + B,
    \end{align*}
    which further implies that
    \begin{align*}
        \frac{f(x+hu)-f(x)}{h} &\leq \frac{(2A + B) + B\|x\|^2}{h} \\
        &\leq \frac{(2A+B)+2B \|x\|^2}{h} + 2B\|x\| + Bh.
    \end{align*}
    Therefore, 
    \begin{align*}
        \frac{f(x+hu)-f(x)}{h} \leq \frac{(2A+B) + 2B\|x\|^2}{h} + 2B\|x\| + Bh
    \end{align*}
    for all $\|x\| > 1,$ $h > 0$, and $u \in \mathbb{S}^{d-1}.$ Choosing $h = \|x\|,$ and using the fact that $\|x\| > 1,$ we have that
	\begin{align*}
		\nabla f(x) \cdot u \leq (2A+B) + 5B \|x\|. 
	\end{align*}
	Next, we apply the first-order bound to $f(x-hu)$:
	\begin{align*}
		f(x-hu) &\geq f(x) - h \nabla f(x) \cdot u,
	\end{align*}
    which implies
    \begin{align*}
		 -\nabla f(x) \cdot u &\leq \frac{f(x-hu) - f(x)}{h}.
	\end{align*}
    A parallel argument shows that
    \begin{align*}
        \frac{f(x-hu) - f(x)}{h} &\leq \frac{(2A+B)+B\|x\|^2}{h} + 2B\|x\| + Bh
    \end{align*}
	for all $\|x\| > 1,$ $h > 0,$ and $u \in \mathbb{S}^{d-1}.$ Choosing $h = \|x\|$ yields
	\begin{align*}
		-\nabla f(x) \cdot u \leq (2A+B) + 5B \|x\|. 
	\end{align*}
	This proves that
	\begin{align*}
		|\nabla f(x) \cdot u| \leq (2A+B) + 5B \|x\|,
	\end{align*}
	for all $\|x\| > 1,$ and since $u$ is arbitrary, we have
	\begin{align*}
		\|\nabla f(x)\| \leq (2A+B) + 5B \|x\|.
	\end{align*}
\end{proof}

\begin{proof}[Proof of Proposition \ref{prop-subg-stability}]
	Clearly it suffices to verify the result in the case of a single head and without the skip connection: given weight matrices $V \in \RR^{k \times d}$ and $A \in \RR^{d \times d}$, we aim to prove that there exist constants $K_1, K_2 > 0$ such that
	\begin{align*}
		\Att_{\theta}(\mu,x) \defeq \frac{\int e^{\langle x, Ay \rangle} Vy \, \mu(dy)}{\int e^{\langle x, Az \rangle} \, \mu(dz)} \leq K_1 + K_2 \|x\|
	\end{align*}
	for all $x \in \RR^d.$ To this end, define the cumulant generating function
	\begin{align*}
		Z_{\mu}(t) \defeq \log \int e^{\langle t,y \rangle} \, \mu(dy),
	\end{align*}
	and note that $Z_{\mu}$ is convex. We can express the map $ \Att_{\theta}(\mu,x)$ in terms of the cumulant generating function:
	\begin{align*}
		\Att_{\theta}(\mu,x) = V \nabla Z_\mu(Ax).
	\end{align*}
	Thus, to show the desired claim, it suffices to prove that $\nabla Z_{\mu}(t)$ grows at most affine-linearly, with constants depending only on the sub-Gaussian parameters of $\mu$. We divide this into two cases.
	
	\paragraph{Case 1: $\|t\| \leq 1$.} We claim that $\nabla Z_{\mu}(t)$ is bounded on the unit ball by a constant depending only on the sub-Gaussian parameters of $\mu.$ Observe that
	\begin{align*}
		\nabla Z_\mu(t) = \frac{\int y e^{\langle y, t \rangle}\, \mu(dy)}{\int e^{\langle y, t \rangle} \, \mu(dy)} = \frac{Z_{\mu}^{(1)}(t)}{Z_{\mu}^{(2)}(t)}.
	\end{align*}
	For the numerator, we have
	\begin{align*}
		Z_{\mu}^{(1)}(t) \leq \left( \int \|y\|^2 \, \mu(dy) \right)^{1/2} \left( \int e^{2 \|y\|} \, \mu(dy) \right)^{1/2}
	\end{align*}
	for all $\|t\| \leq 1.$ By Lemma \ref{lem-subG-compact}, this expression admits an upper bound depending only on the sub-Gaussian parameters $\alpha$ and $\beta$. For the denominator, we have by Jensen's inequality and the Cauchy--Schwarz inequality,
	\begin{align*}
		Z_{\mu}^{(2)}(t) &\geq \exp\left(-\int \|y\| \, \mu(dy) \right),
	\end{align*}
	and by Lemma \ref{lem-subG-compact} the integrand in the exponential can be upper bounded uniformly in $\mu.$ This proves that $\sup_{\|t\| \leq 1} \|\nabla_{\mu} Z(t)\|$ is uniformly bounded over $\mu \in \sg_{\alpha,\beta}(\RR^d).$
	
	\paragraph{Case 2: $\|t\| \geq 1$.}
	Since $\mu$ is sub-Gaussian, it follows (see the book of \citet{vershynin2018high} for equivalent characterizations of sub-Gaussian random variables) that there exist constants $M_1, M_2 > 0$ depending only on $\alpha$ and $\beta$ such that
	\begin{align*}
		Z_{\mu}(t) \leq M_1 + M_2\|t\|^2,
	\end{align*}
	for all $t$ with $\|t\| > 1.$ Since $Z_{\mu}(t)$ is convex, it holds by Lemma \ref{lem: gradientbound} that
	\begin{align*}
		\|\nabla Z_{\mu}(t)\| \leq( 2M_1+M_2) + 5M_2 \|t\|
	\end{align*}
	for all $t$ with $\|t\| \geq 1$ and $\mu \in \sg_{\alpha,\beta}(\RR^d).$ Combining the two cases, we have shown that $ \Att_{\theta}(\mu,x)$ grows affine linearly with constants depending only on the sub-Gaussian parameters of $\mu$ and the attention matrices.
\end{proof}

\subsection{Proofs for Section~\ref{subsec-results-holder}:~\nameref*{subsec-results-holder}}\label{subsec-proofs-holder}
Fix $\alpha \geq 1$ and $\beta > 0,$ and recall that our goal is to bound the Wasserstein distance
\begin{align*}
	\sfW_1\bigl(\Gamma(\mu,\slot)_{\#}\mu, \Gamma(\mu_0,\slot)_{\#}\mu_0\bigr)
\end{align*}
in terms of $\sfW_1(\mu,\mu_0)$, where $\Gamma$ is a mean-field multi-head attention layer. Our proof is based on a truncation argument. For $R > 0,$ we let $\Pi_R\colon \RR^d \rightarrow \RR^d$ denote the projection map onto the ball of radius $R.$ An important step in the proof of Theorem \ref{thm-holder-attn} is the following truncation estimate, which controls the difference between $\Gamma(\mu,x)$ and $\Gamma((\Pi_R)_{\#}\mu,x)$. We first prove the estimate for a single-head attention layer (with the head matrix included in the definition for emphasis), from which the multi-head case follows as a corollary.

\begin{proposition}\label{prop:stab_trunc_attention}
	Define
	\begin{align*}
		\Att_{\theta}(\mu,x) =  W\frac{\int e^{\langle Ax,  y \rangle} V y \, \mu(dy)}{\int e^{\langle Ax,  z \rangle}\, \mu(dz)}.
	\end{align*}
	Then for any $R > 0$ and $M > 0$, it holds that
	\begin{align*}
		&\sup_{\mu\in\sg_{\alpha,\beta}(\RR^d)}\sup_{\norm{x}\leq M}\norm[\big]{\Att_{\theta}(\mu,x)-\Att_{\theta}\bigl((\Pi_R)\push\mu,x\bigr)} \\
		&\qquad\qquad \leq \|W\| \|V\| e^{\|A\| M C_{\alpha,\beta}}(2 + R) \alpha \|A\| M e^{\frac{\beta^2 \alpha^2 \|A\|^2 M^2}{2}}\left( \frac{\|A\| M \beta^2}{R} + 2\beta \right) e^{-\frac{R^2}{2\beta}},
	\end{align*}
	where we have defined
	\begin{align*}
		C_{\alpha,\beta} \defeq \sup_{\nu \in \sg_{\alpha,\beta}(\RR^d)} \int \|x\| \, \nu(dx)<\infty .
	\end{align*}
\end{proposition}
\begin{proof}
	Throughout the proof we will assume that $M > \|A\|^{-1}$. Clearly this assumption is without loss of generality because
	\begin{align*}
		&\sup_{\mu\in\sg_{\alpha,\beta}(\RR^d)}\sup_{\norm{x}\leq M}\norm[\big]{\Att_{\theta}(\mu,x)-\Att_{\theta}\bigl((\Pi_R)\push\mu,x\bigr)} \\ &\qquad\qquad \leq \sup_{\mu\in\sg_{\alpha,\beta}(\RR^d)}\sup_{\norm{x}\leq \max(M,\|A\|^{-1})}\norm[\big]{\Att_{\theta}(\mu,x)-\Att_{\theta}\bigl((\Pi_R)\push\mu,x\bigr)}.
	\end{align*}
	Fix $ \mu\in\sg_{\alpha,\beta}(\RR^d) $ and $ x\in B_M $. Write $\mu_R\defeq (\Pi_R)\push\mu$. Then
	\begin{align*}
		\Att_{\theta}(\mu_R,x)=\frac{\sfN_{\mu,R}(x)}{\sfD_{\mu,R}(x)}\defeq\frac{\int_{\RR^d}WV\Pi_R(y)e^{\ip{\Pi_R(y)}{Ax}}\,\mu(dy)}{\int_{\RR^d}e^{\ip{\Pi_R(z)}{Ax}}\,\mu(dz)}\,.
	\end{align*}
	We similarly write $\Att_{\theta}(\mu,x)=\sfN_\mu(x)/\sfD_\mu(x)$. We estimate
	\begin{align*}
		\norm{\Att_{\theta}(\mu,x)-\Att_{\theta}(\mu_R,x)}\leq \underbrace{\frac{1}{\sfD_\mu(x)}\norm{\sfN_\mu(x)-\sfN_{\mu,R}(x)}}_{\mathrm{(I)}}+\underbrace{\norm{\sfN_{\mu,R}(x)}\abs[\bigg]{\frac{1}{\sfD_\mu(x)}-\frac{1}{\sfD_{\mu,R}(x)}}}_{\mathrm{(II)}}.
	\end{align*}
	
	We begin with term $ \mathrm{(II)} $. Since
	\begin{align*}
		\abs[\bigg]{\ip[\bigg]{\int_{\RR^d} z\,\mu(dz)}{Ax}}\leq \norm{A} M \sup_{\nu\in\sg_{\alpha,\beta}(\RR^d)}\int_{\RR^d}\norm{x}\,\nu(dx)\eqdef \norm{A}M C_{\alpha,\beta}\,,
	\end{align*}
	Jensen's inequality gives 
	\begin{align*}
		\sfD_\mu(x)\geq \exp\biggl(\ip[\bigg]{\int_{\RR^d} z\,\mu(dz)}{Ax}\biggr)\geq \exp(-\norm{A}M C_{\alpha,\beta})\,.
	\end{align*}
	Moreover, since $ \mu_R\in\cP(B_R) $, Lemma \ref{lem-tf-lipschitz} gives $ \norm{\Att_{\theta}(\mu_R,x)}\leq \norm{W} \norm{V}R $. Thus,
	\begin{align*}
		\mathrm{(II)}&=\norm[\bigg]{\frac{\sfN_{\mu,R}(x)}{\sfD_{\mu,R}(x)}}\frac{1}{\sfD_{\mu}(x)}\abs[\big]{\sfD_\mu(x)-\sfD_{\mu,R}(x)}\\
		&\leq \norm{W}\norm{V}R e^{\norm{A}MC_{\alpha,\beta}}\abs[\big]{\sfD_\mu(x)-\sfD_{\mu,R}(x)}\,.
	\end{align*}
	By the mean value theorem, $ \abs{e^s-e^t}\leq e^{\max(s,t)}\abs{s-t} $. Then
	\begin{align*}
		\abs[\big]{\sfD_\mu(x)-\sfD_{\mu,R}(x)}&\leq \int_{\RR^d}\abs[\big]{e^{\ip{y}{Ax}}-e^{\ip{\Pi_R(y)}{Ax}}}\,\mu(dy)\\
		&\leq \int_{\RR^d}e^{\max(\ip{y}{Ax}, \ip{\Pi_R(y)}{Ax})}\abs{\ip{y-\Pi_R(y)}{Ax}}\,\mu(dy)\\
		&\leq \norm{A}M\int_{\RR^d}e^{\norm{A}M\max(\norm{y},R)}\norm{y-\Pi_R(y)}\,\mu(dy)\\
		&\leq \norm{A}M\int_{\norm{y}>R}e^{\norm{A}M\norm{y}}(\norm{y}-R)\,\mu(dy)\,.
	\end{align*}
	We can bound the last line using the exponential uniform integrability implied by Lemma~\ref{lem:ui_exponential_moments}.
	
	Turning to term $ \mathrm{(I)} $, it suffices to bound the difference in numerators because we have already bounded $ \sfD_\mu(x) $ from below uniformly in $ \mu $ and $ x $. Let $ f(y)\defeq ye^{\ip{y}{\lambda}} $. Then $ Df(y)=e^{\ip{y}{\lambda}}(I+y\otimes\lambda) $. Hence, the mean value theorem implies that
	\begin{align*}
		\norm{f(y)-f(z)}\leq e^{\max(\ip{y}{\lambda},\ip{z}{\lambda})}\bigl(1+\max(\norm{y},\norm{z})\norm{\lambda}\bigr)\norm{y-z}\,.
	\end{align*}
	Applying this result gives
	\begin{align*}
		\frac{\norm{\sfN_\mu(x)-\sfN_{\mu,R}(x)}}{\norm{W} \norm{V}}&\leq\int_{\RR^d}\norm[\big]{ye^{\ip{y}{Ax}}-\Pi_R(y)e^{\ip{\Pi_R(y)}{Ax}}}\,\mu(dy) \\
		&\leq 
		\int_{\norm{y}>R}e^{\norm{A}M\norm{y}}\bigl(1+\norm{A}M\norm{y}\bigr)(\norm{y}-R)\,\mu(dy)\,,
	\end{align*}
	which is further bounded by Lemma~\ref{lem:ui_exponential_moments}. Applying Lemma \ref{lem:ui_exponential_moments} to bound the sub-Gaussian integrals and putting together the pieces completes the proof.
\end{proof}

\begin{proposition}\label{prop-tf-truncation}
	Define
	\begin{align*}
		\Gamma(\mu,x) = x + \sum_{h=1}^{H} W_h \frac{\int e^{\langle x, A_h y \rangle} V_h y \, \mu(dy)}{\int e^{\langle x, A_h z \rangle}\, \mu(dz)},
	\end{align*}
	where $A_1, \dots, A_H \in \RR^{d \times d}$, $V_1, \dots, V_H \in \RR^{k \times d}$, and $W_1, \dots, W_H \in \RR^{d \times k}$. Set $c_A = \max_{h \in [H]} \|A_h\|$, $W_h = \max_{h \in [H]} \|W_h\|,$ and $c_V = \max_{h \in [H]} \|V_h\|.$   Then for any $R > 0$ and $M > 0,$ it holds that
	\begin{align*}
		&\sup_{\mu\in\sg_{\alpha,\beta}(\RR^d)}\sup_{\norm{x}\leq M}\norm[\big]{\Gamma(\mu,x)-\Gamma\bigl((\Pi_R)\push\mu,x\bigr)} \\
		&\qquad\qquad \leq H c_W c_V e^{c_A M C_{\alpha,\beta}}(2 + R) \alpha c_A M e^{\frac{\beta^2 \alpha^2 c_A^2 M^2}{2}}\left( \frac{c_A M \beta^2}{R} + 2\beta \right) e^{-\frac{R^2}{2\beta}},
	\end{align*}
	where we have defined
	\begin{align*}
		C_{\alpha,\beta} \defeq \sup_{\nu \in \sg_{\alpha,\beta}(\RR^d)} \int \|x\| \,\nu(dx)<\infty.
	\end{align*}
\end{proposition}
\begin{proof}
	This follows directly from Proposition \ref{prop:stab_trunc_attention}.
\end{proof}

\begin{proof}[Proof of Theorem \ref{thm-holder-attn}]
	Throughout the proof, we use the notation $\mu^r = (\Pi_r)_{\#}\mu$ and $K=\sg_{\al,\beta}(\RR^d)$. We begin with the error decomposition
	\begin{align*}
		\sfW_1(\Gamma(\mu)\push\mu,\Gamma(\mu_0)\push\mu_0)&\leq
		\sfW_1(\Gamma(\mu)\push\mu,\Gamma(\mu^r)\push\mu^r)\\
		&\qquad\qquad + \sfW_1(\Gamma(\mu^r)\push\mu^r,\Gamma(\mu_0^r)\push\mu_0^r)\\
		&\qquad\qquad\qquad\qquad + \sfW_1(\Gamma(\mu_0^r)\push\mu_0^r,\Gamma(\mu_0)\push\mu_0)\\
		&\leq \underbrace{\sfW_1(\Gamma(\mu^r)\push\mu^r,\Gamma(\mu_0^r)\push\mu_0^r)}_{(\mathrm{I})} + 2\underbrace{\sup_{\mu\in K}\sfW_1(\Gamma(\mu)\push\mu,\Gamma(\mu^r)\push\mu^r)}_{\mathrm{(II)}}
	\end{align*}
	for any $r > 0$. For term (I), since $\mu^r$ and $\mu_0^r$ are supported in $B_r,$ Lemma \ref{lem-tf-mtm-lipschitz} guarantees that
	\begin{align*}
		(I) = \sfW_1(\Gamma(\mu^r)\push\mu^r,\Gamma(\mu_0^r)\push\mu_0^r) \leq C_{c_A,c_V,c_W} e^{5 c_A r^2} \sfW_1(\mu^r,\mu_0^r) \leq C_{c_A,c_V,c_W} e^{5 c_A r^2} \sfW_1(\mu,\mu_0),
	\end{align*}
	where the second inequality follows from the fact that $\Pi_r$ is $1$-Lipschitz for any $r > 0.$ We next bound term (II). Since $\norm{\Pi_r(x)}\leq\min(r,\norm{x})\leq\norm{x}$, we have $\mu^r\in \sg_{\alpha,\beta}(\RR^d)$ by \eqref{prop-subg-stability}. Then
	\begin{align*}
		\sfW_1(\Gamma(\mu)\push\mu,\Gamma(\mu^r)\push\mu^r)&\leq \sfW_1(\Gamma(\mu)\push\mu,\Gamma(\mu^r)\push\mu)+\sfW_1(\Gamma(\mu^r)\push\mu,\Gamma(\mu^r)\push\mu^r)\\
		&\leq \norm{\Gamma(\mu)-\Gamma(\mu^r)}_{L^1_\mu} + \mathrm{Lip}(\Gamma(\mu^r))\sfW_1(\mu,\mu^r)\\
		&\leq 
		\sup_{\norm{x}\leq R}\norm{\Gamma(\mu,x)-\Gamma(\mu^r,x)} + \int_{\norm{x}>R}\Bigl(\norm{\Gamma(\mu,x)}+\norm{\Gamma(\mu^r,x)}\Bigr)\, \mu(dx)\\
		&\qquad\qquad + \operatorname{Lip}(\Gamma(\mu^r)) \sfW_1(\mu,\mu^r)\\
		&\leq \sup_{\norm{x}\leq R}\norm{\Gamma(\mu,x)-\Gamma(\mu^r,x)} + 2 C_{\alpha,\beta, V, Q}\int_{\norm{x}>R}\bigl(1+\norm{x}\bigr)\, \mu(dx)\\
		&\qquad\qquad + (1+H c_A c_V c_W r^2) \sfW_1(\mu,\mu^r)
	\end{align*}
	for any $R > 0$, where we used Lemma~\ref{lem-tf-lipschitz} to bound the Lipschitz constant of $\Gamma(\mu^r)$ in the final line. Using Proposition~\ref{prop-tf-truncation} to bound the supremum above, Lemma~\ref{lem:unif_int_m1} to bound the integral $\int_{\|x\| > R} (1+\|x\|)\, \mu(dx),$ and Lemma~\ref{lem:trunc_unif_acc_w1} to bound the truncation error $\sfW_1(\mu,\mu^r),$ we have the following bound for term (II):
	\begin{align*}
		(II) &\leq H c_Vc_W e^{c_A R C_{\alpha,\beta}}(2 + R) \alpha c_A R e^{\frac{\beta^2 \alpha^2 c_A^2 R^2}{2}}\left( \frac{c_A R \beta^2}{r} + 2\beta \right) e^{-\frac{r^2}{2\beta}} \\ &\qquad\qquad+ 2C_{\alpha,\beta,c_V,c_A,c_W}\alpha \left(\frac{\beta}{2R} + 1 + R  \right)e^{-R^2/\beta} + \frac{\alpha\beta}{2r}e^{-r^2/\beta} \\
		&\leq H e^{\frac{C_{\alpha,\beta}^2}{2\alpha^2 \beta^2}} c_V c_W (2+R)\alpha c_A R e^{\beta^2 \alpha^2 c_A^2 R^2} \left( \frac{c_A R \beta^2}{r} + 2\beta \right) e^{-\frac{r^2}{2\beta}} \\ 
        &\qquad\qquad + 2C_{\alpha,\beta,c_V,c_A,c_W}\alpha \left(\frac{\beta}{2R} + 1 + R  \right)e^{-R^2/\beta} + \frac{\alpha\beta}{2r}e^{-r^2/\beta},
	\end{align*}
	where the last inequality follows from Young's inequality. To simplify the rest of this proof, we express this inequality as
	\begin{align*}
		(II) &\leq C_{H,\alpha,\beta,c_A,c_V,c_W} \left(R(1+R)(Rr^{-1}+1) e^{\beta^2 \alpha^2 c_A^2 R^2-\frac{r^2}{2\beta}} + (1+r+r^{-1})\bigl(e^{-r^2/\beta}+e^{-R^2/\beta}\bigr) \right),
	\end{align*}
	where the constant $C_{H,\alpha,\beta,c_A,c_V,c_W}$ may change from line to line throughout the rest of the proof. To conclude the proof, we must choose $r$ and $R$ as functions of $\sfW_1(\mu,\mu_0).$ We first set
	\begin{align*}
		R = \sqrt{\frac{1}{2(\beta^3\alpha^2 c_A^2+1)}}r,
	\end{align*}
	so that the upper bound on term (II) becomes
	\begin{align*}
		(II) &\leq C_{H,\alpha,\beta,c_A,c_V,c_W} (1+r^2+r^{-1}) e^{-\min \left(\frac{1}{2(\beta^3\alpha^2 c_A^2 + 1)}, \frac{1}{\beta} \right)r^2}.
	\end{align*}
	To balance terms (I) and (II), we set
    \begin{align*}
        r = \sqrt{\frac{1}{k_1+k_2} \log_{+}\left(\sfW_1(\mu,\mu_0)^{-1} \right)},
    \end{align*}
    where
    \begin{align}\label{eq-app-k1k22}
        k_1 \defeq 5c_A\qa k_2 \defeq \min \left(\frac{1}{\beta}, \frac{1}{2(\beta^3 \alpha^2 c_A^2 + 1)} \right).
    \end{align}
	This leads to the estimate
	\begin{align*}
		&\sfW_1(\Gamma(\mu,\slot)_{\#}\mu,\Gamma(\mu_0,\slot)_{\#}\mu_0) \\
        &\qquad\qquad \leq C_{H,\alpha,\beta,c_A,c_V,c_W} \Bigl(1+\max\bigl(\log(\sfW_1(\mu,\mu_0)), \log_+(\sfW_1(\mu,\mu_0)^{-1})\bigr)\Bigr) \sfW_1(\mu,\mu_0)^{\frac{k_2}{k_1+k_2}},
	\end{align*}
	  where $k_1$ and $k_2$ are defined in Equation \eqref{eq-app-k1k22}. To conclude, we remove the log-factor at the expense of enlarging the constant factor and the H\"older estimate. This gives the final estimate
	\begin{align*}
		\sfW_1(\Gamma(\mu,\slot)_{\#}\mu,\Gamma(\mu_0,\slot)_{\#}\mu_0) \leq C_{H,\alpha,\beta,c_A,c_V,c_W}\sfW_1(\mu,\mu_0)^{\gamma}
	\end{align*}
	for $\gamma \in (0,\frac{k_2}{k_1+k_2})$. We conclude the proof.
\end{proof}

\begin{proof}[Proof of Theorem \ref{thm-holder-tf}]
	We will prove the statement by induction on the number of layers $L.$ For the base case, assume that $\mathsf{T}$ is an in-context mapping given by $\mathsf{T}(\mu,x) = \phi(\Gamma(\mu,x)),$ where $\Gamma$ is a multi-head self-attention layer and $\phi$ is a ReLU feedforward neural network. The mapping $\mu \mapsto \Gamma(\mu,\slot)_{\#}\mu$ is H\"older continuous on $\sg_{\alpha,\beta}(\RR^d),$ and the mapping $\mu \mapsto \phi_{\#}\mu$ is globally Lipschitz (since $\phi$ is a Lipschitz mapping). Therefore, the mapping $\mu \mapsto \mathsf{T}(\mu)\push\mu$ is H\"older continuous on $\sg_{\alpha,\beta}(\RR^d);$ this establishes the base case. For the induction step, let $\mathsf{T}$ be an in-context composition of $L$ multi-head self-attention layers and ReLU feedforward layers. Then we can write
	\begin{align*}
		\mathsf{T} = (\mathsf{T}_L \diamond \mathsf{T}_{L-1}),
	\end{align*}
	where $\mathsf{T}_{L-1}$ is an in-context composition of $L-1$ multi-head self-attention layers and ReLU feedforward layers, and $\mathsf{T}_L = \phi_L \diamond \Gamma_L$ is a multi-head self attention layer composed with a ReLU feedforward layer. By induction, there exist constants $C > 0$ and $\gamma \in (0,1)$ such that $G_{L-1}$ satisfies the estimate
	\begin{align*}
		\sfW_1(\mathsf{T}_{L-1}(\mu,\slot)_{\#}\mu, \mathsf{T}_{L-1}(\mu_0,\slot)_{\#}\mu_0) \leq C \sfW_1(\mu,\mu_0)^{\gamma}
	\end{align*}
	for all $\mu$ and $ \mu_0$ in $ \sg_{\alpha,\beta}(\RR^d).$ In addition, by Theorem \ref{thm-subG-stable}, there exist constants $\alpha' \geq 1$ and $\beta' > 0$, depending only on $\alpha,$ $\beta,$ and $\mathsf{T}_{L-1},$ such that $\mathsf{T}_{L-1}(\mu,\slot)_{\#}\mu \in \sg_{\alpha',\beta'}(\RR^k)$ for all $\mu \in \sg_{\alpha,\beta}(\RR^d).$ Since the base case argument guarantees that $\mathsf{T}_L$ is H\"older continuous on $\sg_{\alpha',\beta'}(\RR^k)$ (say, with constant $C' > 0$ and exponent $\gamma' \in (0,1)$), it follows that
	\begin{align*}
		\sfW_1(\mathsf{T}(\mu,\slot)_{\#}\mu, \mathsf{T}(\mu_0,\slot)_{\#}\mu_0)
		&= \sfW_1((\mathsf{T}_L \diamond \mathsf{T}_{L-1})(\mu,\slot)_{\#}\mu, (\mathsf{T}_L \diamond \mathsf{T}_{L-1})(\mu_0,\slot)_{\#}\mu_0) \\
		&\leq C' \sfW_1(\mathsf{T}_{L-1}(\mu,\slot)_{\#}\mu, \mathsf{T}_{L-1}(\mu_0,\slot)_{\#}\mu_0)^{\gamma} \\
		&\leq C C' \sfW_1(\mu,\mu_0)^{\gamma \gamma'}.
	\end{align*}
	We conclude the proof.
\end{proof}

\subsection{Proofs for Section~\ref{subsec-results-samples}:~\nameref*{subsec-results-samples}}\label{subsec-proofs-samples}
\begin{proof}[Proof of Theorem~\ref{thm-sample-tf}]
	Throughout this proof, $C$ denotes a positive constant---depending only on $\alpha$, $\beta$, and the parameters of $\mathsf{T}$---whose precise value may differ from line to line. Fix $\mu \in \sg_{\alpha,\beta}(\RR^d)$, let $x_1, \dots, x_N \sim \mu$ be iid, and let $\mu_N = \frac{1}{N} \sum_{i=1}^{N} \delta_{x_i}.$ Let $\mathcal{A}_N$ denote the event on which $\mu_N \in \sg_{3+\alpha/7, 7\beta}.$ By Proposition \ref{prop-subG-empirical}, there exists a $N_{\alpha,\beta} \in \mathbb{N}$ such that
	\begin{align*}
		\mathbb{P}(\mathcal{A}_N) \geq 1-2\alpha N^{-2}\qfa N \geq N_{\alpha,\beta}.
	\end{align*}
	Notice that 
    \begin{align*}
            \sg_{\alpha,\beta}(\RR^d) \bigcup \sg_{3+\alpha/7,7\beta}(\RR^d) \subset \sg_{\max(\alpha,3+\alpha/7),7\beta}(\RR^d).
    \end{align*}
    By Theorem~\ref{thm-holder-tf}, the mapping $\mu \mapsto \mathsf{T}(\mu,\slot)_{\#}\mu$ is H\"older continuous on $\sg_{\max(\alpha,3+\alpha/7),7\beta}(\RR^d)$ with H\"older exponent $\gamma = \gamma(\alpha,\beta,\sfT) \in (0,1)$ and constant $C = C(\alpha,\beta,\sfT) > 0$ depending only on $\alpha$, $\beta,$ and the parameters of $\sfT.$ We then proceed to estimate
	\begin{align}\label{eq-bad-set-decomp}
        \begin{split}
		&\E \sfW_1(\mathsf{T}(\mu,\slot)_{\#}\mu,\mathsf{T}(\mu_N,\slot)_{\#}\mu_N) = \E\left[\sfW_1(\mathsf{T}(\mu,\slot)_{\#}\mu,\mathsf{T}(\mu_N,\slot)_{\#}\mu_N)(\mathbf{1}_{\mathcal{A}_N} + \mathbf{1}_{\mathcal{A}_N^c})\right] \\
		&\qquad\qquad \leq C_G\E[\sfW_1(\mu,\mu_N)^{\gamma}] + \E[\sfW_1^2(\mathsf{T}(\mu,\slot)_{\#}\mu,\mathsf{T}(\mu_N,\slot)_{\#}\mu_N)]^{1/2} \slot \sqrt{\mathbb{P}(\mathcal{A}_N^c)} \\
		&\qquad\qquad\leq C_G \E[\sfW_1(\mu,\mu_N)]^{\gamma} + \E[\sfW_1^2(\mathsf{T}(\mu,\slot)_{\#}\mu,\mathsf{T}(\mu_N,\slot)_{\#}\mu_N)]^{1/2} \slot \frac{\sqrt{2\alpha}}{N},
        \end{split}
	\end{align}
	where the second line follows from the Cauchy--Schwarz inequality and the final line follows from Jensen's inequality. To estimate the second term, we have
	\begin{align*}
		\E[\sfW_1^2(\mathsf{T}(\mu,\slot)_{\#}\mu,\mathsf{T}(\mu_N,\slot)_{\#}\mu_N)] &\leq 2 \E[\sfW_1^2(\mathsf{T}(\mu,\slot)_{\#}\mu, \mathsf{T}(\mu,\slot)_{\#}\mu_N)] + 2\E[\sfW_1^2(\mathsf{T}(\mu,\slot)_{\#}\mu_N,\mathsf{T}(\mu_N,\slot)_{\#}\mu_N)] \\
		&\eqdef I + II.
	\end{align*}
	For term $I,$ if $\gamma_N$ denotes any coupling of $\mu$ and $\mu_N$, we have
	\begin{align*}
		\E[\sfW_1^2(\mathsf{T}(\mu,\slot)_{\#}\mu,\mathsf{T}(\mu,\slot)_{\#}\mu_N)] &\leq \E[\sfW_2^2(\mathsf{T}(\mu,\slot)_{\#}\mu,\mathsf{T}(\mu,\slot)_{\#}\mu_N)] \\
		&\leq \E \left[\int \|\mathsf{T}(\mu,x) - \mathsf{T}(\mu,x')\|^2 \,\gamma_N(dx,dx') \right] \\
		&\leq 2 \E \left[\int \left(\|\mathsf{T}(\mu,x)\|^2 + \mathsf{T}(\mu,x')\|^2 \right)\,\gamma_N(dx,dx_N) \right] \\
		&\leq 2C \E \left[\int(1 + \|x\|^2 + \|x'\|^2)\,\gamma_N(dx,dx') \right] \\
		&\leq 2C \left(1+ \int \|x\|^2 \, \mu(dx) + \E \left[ \frac{1}{N} \sum_{i=1}^{N} \|x_i\|^2 \right] \right) \\
		&= 2C \left(1 + 2 \int \|x\|^2 \, \mu(dx) \right) \\
		&= O_{\alpha,\beta,\sfT}(1),
	\end{align*}
	where the third line applies Lemma \ref{lem-subG-pushforward} to deduce that $\|\mathsf{T}(\mu,x)\|^2 + \|\mathsf{T}(\mu,x')\|^2 \lesssim 1 + \|x\|^2 + \|x'\|^2$, and the last line applies Lemma \ref{lem:unif_int_m1}. This shows that term $I$ can be bounded by a constant independent of $N$ and depending only on $\alpha$ and $\beta$ and the parameters of $\sfT$. For term $II,$ we have
	\begin{align*}
		\E[\sfW_1^2(\mathsf{T}(\mu,\slot)_{\#}\mu_N,\mathsf{T}(\mu_N,\slot)_{\#}\mu_N)] &\leq \E[\sfW_2^2(\mathsf{T}(\mu,\slot)_{\#}\mu_N,\mathsf{T}(\mu_N,\slot)_{\#}\mu_N)] \\
		&\leq \E \left[ \frac{1}{N} \sum_{i=1}^{N} \|\mathsf{T}(\mu,x_i) - \mathsf{T}(\mu_N,x_i)\|^2 \right] \\
		&\leq 2 \E \left[ \frac{1}{N} \sum_{i=1}^{N} \left(\|\mathsf{T}(\mu,x_i)\|^2 + \|\mathsf{T}(\mu_N,x_i)\|^2 \right) \right] \\
		&\leq 2 \E \left[\frac{1}{N} \sum_{i=1}^{N} C (1+ \|x_i\|^2) + \frac{1}{N} \sum_{i=1}^{N}\|\mathsf{T}(\mu_N,x_i)\|^2 \right].
	\end{align*}
	Lemma \ref{lem-empirical-tf-bound} proves that 
	\begin{align*}
		\max_{i \in N} \|\mathsf{T}(\mu_N,x_i)\| \leq C \max_{n \in [N]} \|x_n\|,
	\end{align*}
	where the constant $C$ depends only on the parameters of $\sfT.$ This proves that term $II$ satisfies the bound
	\begin{align*}
		II &\leq C \E\left[1+\max_{i \in [N]} \|x_i\|^2\right] \\
		&\leq C\bigl(1+\log(N)\bigr),
	\end{align*}
	where we applied Lemma \ref{lem-expected-subG-max} in the second line. Combining the estimates for terms $I$ and $II,$ we conclude that
	\begin{align*}
		\E[\sfW_1^2(\mathsf{T}(\mu,\slot)_{\#}\mu,\mathsf{T}(\mu_N,\slot)_{\#}\mu_N)]^{1/2} \leq C\sqrt{1+\log(N)}.
	\end{align*}
	Plugging this bound into \eqref{eq-bad-set-decomp}, we obtain
	\begin{align*}
		\E[\sfW_1(\mathsf{T}(\mu,\slot)_{\#}\mu, \mathsf{T}(\mu_N,\slot)_{\#}\mu_N)] \leq C\left( \E[\sfW_1(\mu,\mu_N)]^{\gamma} + \frac{\sqrt{1+\log(N)}}{N}\right).
	\end{align*}
	To conclude, Proposition \ref{prop-subG-convergence} guarantees that
	\begin{align*}
		\E[\sfW_1(\mu,\mu_N)] \leq C \sqrt{\log(N)\left(1+\log(1+N^{1/d}) \right)}N^{-1/d},
	\end{align*}
	where the implicit constant depends only on $d$, $\alpha,$ and $\beta.$ In turn, this proves that
	\begin{align*}
		\E\left[\sfW_1(\mathsf{T}(\mu,\slot)_{\#}\mu, \mathsf{T}(\mu_N,\slot)_{\#}\mu_N) \right] &\leq C \left(\log(N)^{\gamma/2}\left(1+\log(1+N^{1/d})\right)^{\gamma/2}N^{-\gamma/d} + \frac{\sqrt{1+\log(N)}}{N} \right) \\
		&\leq C \log(N)^{\gamma/2}\left(1+\log(1+N^{1/d})\right)^{\gamma/2}N^{-\gamma/d},
	\end{align*}
	as was to be shown.
\end{proof}
\begin{proof}[Proof of Proposition \ref{thm-sample-tf-compact}]
	Fix $\mu \in \mathcal{P}(B_R).$ Lemma \ref{lem-tf-mtm-lipschitz}, along with the fact that ReLU feedforward networks are globally Lipschitz continuous, guarantees that the mapping $\mu \mapsto \mathsf{T}(\mu,\slot)_{\#} \mu$ is Lipschitz continuous on $\mathcal{P}(B_R)$ with Lipschitz constant $C_{R,\mathsf{T}}$ depending on $R$ and the parameters of $\mathsf{T}.$ Therefore,
	\begin{align*}
		\E \sfW_1(\mathsf{T}(\mu,\slot)_{\#}\mu,\mathsf{T}(\mu_N,\slot)_{\#}\mu_N) \leq C_{R,\mathsf{T}} \E \sfW_1(\mu,\mu_N).
	\end{align*}
	Then, by Proposition \ref{prop: uniformquantization}, there is a constant $C_d$ depending on the dimension $d$ such that
	\begin{align*}
		\E \sfW_1(\mu,\mu_N) \leq C_d R \left(1+\sqrt{\log(1+N^{1/d})} \right)N^{-1/d}.
	\end{align*}
	Combining these two inequalities gives the desired result. 
\end{proof}

\subsection{Proofs for Section~\ref{subsec-results-cross}:~\nameref*{subsec-results-cross}}\label{subsec-proofs-cross}
\begin{proof}[Proof of Theorem~\ref{thm: holder-cross-attn}]
	We seek to bound
	\begin{align*}
		\sfW_1(\Gamma(\mu,\slot)_{\#}\eta, \Gamma(\mu_0,\slot)_{\#}\eta_0),
	\end{align*}
	where $\mu,\mu_0 \in \sg_{\alpha,\beta}(\RR^d)$ and $\eta,\eta_0 \in \sg_{\alpha,\beta}(\RR^{d'})$ for $\alpha \geq 1$ and $\beta > 0.$ We decompose this as
	\begin{align*}
		\sfW_1(\Gamma(\mu,\slot)_{\#}\eta, \Gamma(\mu_0,\slot)_{\#}\eta_0) &\leq \sfW_1(\Gamma(\mu,\slot)_{\#}\eta, \Gamma(\mu_0,\slot)_{\#}\eta) + \sfW_1(\Gamma(\mu_0,\slot)_{\#}\eta, \Gamma(\mu_0,\slot)_{\#}\eta_0) \\
		&\eqdef I + II.
	\end{align*}
	For term $I,$ we introduce auxiliary parameters $R>0$ and $r > 0$, and further decompose it as
	\begin{align*}
		I &\leq \int \|\Gamma(\mu,z) - \Gamma(\mu_0,z)\| \,\eta(dz) \\
		&\leq \int \Bigl(\|\Gamma(\mu,z) - \Gamma(\mu^r,z)\| + \|\Gamma(\mu^r, z) - \Gamma(\mu_0^r,z)\| + \|\Gamma(\mu_0^r,z) - \Gamma(\mu_0,z)\|  \Bigr)\, \eta(dz) \\
		&\leq 2 \sup_{\nu \in \sg_{\alpha,\beta}(\RR^{d}), \|z\| \leq R} \|\Gamma(\nu,z)-\Gamma(\nu^r,z)\| + \sup_{\|z\| \leq R} \|\Gamma(\mu^r,z)-\Gamma(\mu_0^r,z)\| \\ 
        &\qquad\qquad + 3\sup_{\nu,\nu_0 \in \sg_{\alpha,\beta}(\RR^{d}), \rho \in \sg_{\alpha,\beta}(\RR^{d'})} \int_{\norm{z}>R}\|\Gamma(\nu,z) - \Gamma(\nu_0,z)\| \rho(dz) \\
		&\eqdef I_I + I_{II} + I_{III}.
	\end{align*}
	By Proposition \ref{prop-tf-truncation}, term $I_I$ is bounded by
	\begin{align*}
		I_I &\leq 2H c_V c_W e^{c_A R C_{\alpha,\beta}}(2 + r) \alpha c_A R e^{\frac{\beta^2 \alpha^2 c_A^2 R^2}{2}}\left( \frac{c_A R \beta^2}{r} + 2\beta \right) e^{-\frac{r^2}{2\beta}} \\
        &\leq C_{H,c_A,c_V,c_W,\alpha,\beta} (1+r)R(1+Rr^{-1}) e^{\beta^2 \alpha^2 c_A^2 R^2 - \frac{r^2}{2\beta}},
	\end{align*}
	where the last line is an application of Young's inequality. For term $I_{II},$ we apply Corollary \ref{cor-mha-lipschitz}, which bounds the Lipschitz constant of the mapping $\nu \mapsto \Gamma(\nu,z)$ for $\nu \in \mathcal{P}(B_r)$ and $z \in B_R$, to obtain
	\begin{align*}
		I_{II} &\leq H e^{4c_A Rr} \max(c_Vc_W(1+c_A r R), c_A c_Vc_W r R) \sfW_1(\mu^r, \mu_0^r) \\
		&\leq H e^{4c_A Rr} \max(c_Vc_W(1+c_A r R), c_A c_Vc_W r R) \sfW_1(\mu, \mu_0),
	\end{align*}
	where the last line follows from the fact that the truncation map is $1$-Lipschitz for each $r > 0,$ and hence the mapping $\nu \mapsto \nu^r$ is $1$-Lipschitz with respect to $\sfW_1.$ For term $I_{III},$ we use the linear growth of $\Gamma(\nu,z)$ with respect to $z$ and the tail decay of $\rho$; this yields
	\begin{align*}
		I_{III} &\leq 3\sup_{\nu,\nu_0, \in \sg_{\alpha,\beta}(\RR^d), \rho \in \sg_{\alpha,\beta}(\RR^{d'})} \int_{\norm{z}>R} (\|\Gamma(\nu,z) + \Gamma(\nu_0,z)\| \, \rho(dz) \\
		&\leq 3 C_{\alpha,\beta} \sup_{\rho \in \sg_{\alpha,\beta}(\RR^{d'})} \int_{\norm{z}>R}(1+\|z\|) \, \rho(dz) \\
		&\leq 3 C_{\alpha,\beta} \alpha \left(1 + \frac{\beta}{2R} + R\right)e^{-R^2/\beta},
	\end{align*}
	where the second line follows from Proposition~\ref{prop-subg-stability} and the last line follows from Lemma~\ref{lem:unif_int_m1}. Summing the estimates of terms $I_I,$ $I_{II}$, and $I_{III}$ furnishes a bound on term $I.$ 
    
    Turning to term $II,$ we let $\Pi$ denote the optimal coupling between $\eta$ and $\eta_0.$ Then we have
	\begin{align*}
		II &\leq \int \|\Gamma(\mu_0,z)-\Gamma(\mu_0,z')\|\, \Pi(dz,dz') \\
		&\leq \int \|\Gamma(\mu_0,z) - \Gamma(\mu_0^r,z)\| \,\eta(dz) + \int \|\Gamma(\mu_0,z')-\Gamma(\mu_0^r,z')\| \,\eta_0(dz) + \int \|\Gamma(\mu_0^r,z) - \Gamma(\mu_0^r,z')\| \,\Pi(dz,dz') \\
		&\leq 2 \sup_{\nu \in \sg_{\alpha,\beta}(\RR^d), \|z\| \leq R} \|\Gamma(\nu,z) - \Gamma(\nu^r,z)\| + 2 \sup_{\nu \in \sg_{\alpha,\beta}(\RR^d), \rho \in \sg_{\alpha,\beta}(\RR^{d'})} \int_{\norm{z}>R} \|\Gamma(\nu,z) - \Gamma(\nu^r,z)\|\rho(dz) \\ 
		&\qquad\qquad+ \int \|\Gamma(\mu_0^r,z) - \Gamma(\mu_0^r,z')\| \Pi(dz,dz') \\
        &\eqdef  II_I + II_{II} + II_{III}.
	\end{align*}
	Terms $II_I$ and $II_{II}$ can be bounded using the estimates of terms $I_I$ and $I_{III}:$
	\begin{align*}
		II_I &\leq 2H c_V e^{c_A R C_{\alpha,\beta}}(2 + r) \alpha c_A R e^{\frac{\beta^2 \alpha^2 c_A^2 R^2}{2}}\left( \frac{c_A R \beta^2}{r} + 2\beta \right) e^{-\frac{r^2}{2\beta}} \\
        &\leq C_{H,c_A,c_V,c_W,\alpha,\beta}(1+r)R(1+Rr^{-1}) e^{\beta^2\alpha^2c_A^2 R^2-\frac{r^2}{2\beta}}
	\end{align*}
    and
    \begin{align*}
        II_{II} \leq 2 C_{\alpha,\beta} \alpha \left(1 + \frac{\beta}{2R} + R\right)e^{-R^2/\beta}.
    \end{align*}
	To bound term $II_{III},$ we can apply Corollary \ref{cor-mha-lipschitz}, which guarantees that the mapping $z \mapsto \Gamma(\nu,z)$ is $(1+H c_A c_V r^2)$-Lipschitz for any $\nu \in \mathcal{P}(B_r).$ This implies that
	\begin{align*}
		II_{III} &\leq (1+H c_A c_V r^2) \int \|z-z'\| \, \Pi(dz,dz') \\
		&= (1+H c_A c_V r^2) \sfW_1(\eta,\eta_0).
	\end{align*}
	To finish the proof, let us assume that $r \geq 1$ and $R \geq 1$; hence, the estimates $1/r \leq 1$ and $1/R \leq 1$ hold. Combining estimates for terms $I$ and $II$, and applying Young's inequality as in the proof of Theorem \ref{thm-holder-attn}, we have
	\begin{align*}
		\sfW_1(\Gamma(\mu,\slot)_{\#}\eta,\Gamma(\mu_0,\slot)_{\#}\eta_0) &\leq C \Bigl( \underbrace{rR^2e^{\beta^2 \alpha^2 c_A^2 R^2 - r^2/(2\beta)}}_{I_I + II_I} + \underbrace{Re^{-R^2/\beta}}_{I_{III} + II_{II}} + \underbrace{rR e^{4c_A r R}\sfW_1(\mu,\mu_0)}_{I_{II}} + \underbrace{r^2 \sfW_1(\eta,\eta_0)}_{II_{III}}\Bigr),
	\end{align*}
	where $C$ denotes a constant which depends only on $\alpha$, $\beta,$ $H$, $c_A$, $c_V$, and $c_W,$ and which may change from line to line. To finish the proof, we first choose $R$ as a suitable function of $r$, and then choose $r$ as a suitable function of $\sfW_1(\mu,\mu_0)$. By choosing 
	\begin{align*}
		R = \sqrt{\frac{1}{2(\beta^3\alpha^2 c_A^2+1)}}r,
	\end{align*}
	the upper bound on $\sfW_1(\Gamma(\mu,\slot)_{\#}\eta, \Gamma(\mu_0,\slot)_{\#}\eta_0)$ simplifies to
	\begin{align*}
		\sfW_1(\Gamma(\mu,\slot)_{\#}\eta,\Gamma(\mu_0,\slot)_{\#}\eta_0) &\leq C \left(r^3 e^{-\frac{r^2}{2\beta(\beta^3\alpha^2 c_A^2 + 1)}} + r^2e^{2\sqrt{\frac{2c_A^2}{\beta^3\alpha^2 c_A^2+1}}r^2} \sfW_1(\mu,\mu_0) + r^2\sfW_1(\eta,\eta_0) \right)
	\end{align*}
	up to a change in the value of $C$ from line to line. We write 
    \begin{align}\label{eq-app-k1k2}
        k_1 \defeq \frac{1}{2\beta(\beta^3 \alpha^2 c_A^2 + 1)}\qa k_2 \defeq 2 \sqrt{\frac{2c_A^2}{\beta^3 \alpha^2 c_A^2 + 1}} 
    \end{align}
    for the exponents above. We then set
	\begin{align*}
		r = \sqrt{\frac{1}{k_1+k_2}\log(\sfW_1(\mu,\mu_0)^{-1})}.
	\end{align*}
    Under the assumption that $r\geq 1$ and $ R \geq 1,$ this leads to a final estimate of the form
	\begin{align*}
		\sfW_1(\Gamma(\mu,\slot)_{\#}\mu, \Gamma(\mu_0,\slot)_{\#}\mu_0) &\leq C\log(\sfW_1(\mu,\mu_0)^{-1})^{3/2} \left(\sfW_1(\mu,\mu_0)^{\gamma} + \sfW_1(\eta,\eta_0) \right),
	\end{align*}
    where $\gamma \defeq \frac{k_2}{k_1+k_2},$ with $k_1$ and $k_2$ defined in~\eqref{eq-app-k1k2}. Notice that there exists a sufficiently small $ M= M(\alpha,\beta, H,c_A, c_V)>0$ such that if $\sfW_1(\mu,\mu_0)\leq M$, then the estimates $r \geq 1$ and $R \geq 1$ indeed hold. On the other hand, if $\sfW_1(\mu,\mu_0)> M$, then Proposition~\ref{prop-subg-stability} implies that
	\begin{align*}
		\sfW_1\bigl(\Gamma(\mu,\slot)_{\#}\eta, \Gamma(\mu_0,\slot)_{\#}\eta_0\bigr)&\leq \sfW_1\bigl(\Gamma(\mu,\slot)_{\#}\eta, \delta_0\bigr) + \sfW_1\bigl(\delta_0, \Gamma(\mu_0,\slot)_{\#}\eta_0\bigr)\\
		&\leq \int \norm{\Gamma(\mu,x)}\,\eta(dx) + \int \norm{\Gamma(\mu_0,x)}\,\eta_0(dx)\\
		&\leq 2C(\alpha,\beta, H,c_A, c_V)\sup_{\nu\in\sg_{\al,\beta}(\RR^{d'})}\int \bigl(1+\norm{x}\bigr)\,\nu(dx)\eqdef C'(\alpha,\beta, H,c_A, c_V)\,.
	\end{align*}
	It therefore holds that
	\begin{align*}
		\sfW_1\bigl(G(\mu,\slot)_{\#}\eta, G(\mu_0,\slot)_{\#}\eta_0\bigr)&\leq C'M^{-\gamma}\sfW_1(\mu,\mu_0)^{\gamma}\\
		&\leq C'M^{-\gamma}\bigl(\sfW_1(\mu,\mu_0)^{\gamma}+\sfW_1(\eta,\eta_0)\bigr)\\
		&\leq C'M^{-\gamma}\max \bigl(1,\log_+(\sfW_1(\mu,\mu_0)^{-1})^{3/2}\bigr)\bigl(\sfW_1(\mu,\mu_0)^{\gamma}+\sfW_1(\eta,\eta_0)\bigr).
	\end{align*}
	Enlarging constants completes the proof.
\end{proof}

\begin{proof}[Proof of Theorem \ref{thm-sample-cross-attn}]
	\textbf{Case 1: $\mu \in \mathcal{P}(B_R)$ and $\eta \in \sg_{\alpha,\beta}(\RR^d)$:}
	Fix $\mu \in \mathcal{P}(B_R(\RR^d))$ and $\eta \in \sg_{\alpha,\beta}(\RR^d).$ By the triangle inequality, we have
	\begin{align*}
		\E \sfW_1(\Gamma(\mu,\slot)_{\#}\eta, \Gamma(\mu_N,\slot)_{\#}\eta_M) &\leq \E \sfW_1(\Gamma(\mu,\slot)_{\#}\eta, \Gamma(\mu_N,\slot)_{\#}\eta) + \E\sfW_1(\Gamma(\mu_N,\slot)_{\#}\eta, \Gamma(\mu_N,\slot)_{\#}\eta_M) \\
		&\eqdef I + II.
	\end{align*}
	To bound term $I,$ we introduce an auxiliary parameter $r > 0$ and use the estimate
	\begin{align*}
		I &\leq \E \int \|\Gamma(\mu,x) -\Gamma(\mu_N,x)\| \, \eta(dx) \\
		&\leq \E \sup_{\|x\| \leq r} \|\Gamma(\mu,x) - \Gamma(\mu_N,x)\| + \E \int_{\|x\| > r} \|\Gamma(\mu,x) - \Gamma(\mu_N,x)\| \,\eta(dx) \\
		&\eqdef I_I + I_{II}.
	\end{align*}
	To bound term $I_{I},$ we apply Corollary \ref{cor-mha-lipschitz}, which controls the Lipschitz constant of the mapping $\mu \mapsto \Gamma(\mu,x)$ for $x \in B_r$ and $\mu \in \mathcal{P}(B_R)$:
	\begin{align*}
		I_I &\leq H e^{4 c_A r^2} \max(c_Vc_W(1+ c_A rR), c_A c_Vc_W rR) \E \sfW_1(\mu,\mu_N) \\
		&\leq  H e^{4 c_A r^2} \max(c_Vc_W(1+ c_A rR), c_A c_Vc_W rR) C_d R \left(1+\sqrt{\log(1+N^{1/d})} \right)N^{-1/d},
	\end{align*}
	where the last line follows from Wasserstein law of large numbers stated in Proposition \ref{prop: uniformquantization}. 
    
    To bound term $I_{II},$ note that since both $\mu$ and $\mu_N$ belong to $ \mathcal{P}(B_R),$ it holds by Lemma \ref{lem-compact-subg-embed} that there exists an $\alpha_R \geq 1$ and $\beta_R > 0$ such that $\mu$ and $\mu_N$ belong to $\sg_{\alpha_R,\beta_R}(\RR^d).$ Using Proposition \ref{prop-subg-stability}, it follows that there exists a constant $C$ depending on $\alpha_R$ and $\beta_R$ such that
	\begin{align*}
		I_{II} &\leq \E \int_{\|x\| > r} (\|\Gamma(\mu,x)\| + \|\Gamma(\mu_N,x)\|)\, \eta(dx) \\
		&\leq C \int_{\|x\| > r} (1+\|x\|)\, \eta(dx) \\
		&\lesssim e^{-r^2/\beta},
	\end{align*}
	where the last line follows from Lemma \ref{lem:unif_int_m1}. 
    
    For term $II$, Corollary \ref{cor-mha-lipschitz} guarantees that the Lipschitz constant of the map $\Gamma(\nu,\slot)$ is bounded by $(1+H c_A c_V R^2)$ for any $\nu \in \mathcal{P}(B_R)$. It follows from Proposition \ref{prop-subG-convergence} that
	\begin{align*}
		II & \leq (1+H c_A c_Vc_W R^2) \E \sfW_1(\eta,\eta_M) \\
		&\leq (1+H c_A c_Vc_W R^2) C_{d,\alpha,\beta} \sqrt{\log(M)\left(1+\log(1+M^{1/d}) \right)}M^{-1/d}
	\end{align*}
    for all $M\in\N$. To finish the proof, we set the value of $r$ as
	\begin{align*}
		r = \sqrt{\frac{\beta}{d(4 c_A \beta + 1)}\log(N)}, 
	\end{align*}
	which leads to the estimates
	\begin{align*}
		e^{4c_A r^2} N^{-1/d} = e^{-r^2/\beta} = N^{-\frac{1}{d(4 c_A \beta + 1)}}.
	\end{align*}
	This concludes the proof of the first case.
	
	\textbf{Case 2: $\mu \in \mathcal{P}(B_R)$ and $\eta \in \mathcal{P}(B_R)$:} Our proof is similar to the proof of results from work of \citet[Section~4.2]{bohbot2025token} and involves controlling the covering numbers of certain function classes associated with multi-head attention layers. Fixing $\mu\in\mathcal{P}(B_R)$ and $ \gamma \in \mathcal{P}(B_R)$, it holds that
	\begin{align*}
		\E \sfW_1(\Gamma(\mu,\slot)_{\#}\eta, \Gamma(\mu_N,\slot)_{\#}\eta_M) &\leq \E \sfW_1(\Gamma(\mu,\slot)_{\#}\eta,\Gamma(\mu_N,\slot)_{\#}\eta) + \E \sfW_1(\Gamma(\mu_N,\slot)_{\#}\eta,\Gamma(\mu_N,\slot)_{\#}\eta_M) \\
		&\eqdef I + II.
	\end{align*}
	For term $I,$ we have the bound
	\begin{align*}
		I &\leq \E \sup_{x \in B_R} \|\Gamma(\mu,x) - \Gamma(\mu_N,x)\| \\
		&\leq Hc_W \slot \max_{\|A\| \leq c_A, \|V\| \leq c_V} \E \sup_{x \in B_R} \left\|\frac{\int e^{\langle x, Ay \rangle} Vy \, \mu(dy)}{\int e^{\langle x, Ay \rangle} \, \mu(dy)} - \frac{\int e^{\langle x, Ay \rangle} Vy \,\mu_N(dy)}{\int e^{\langle x, Ay \rangle}\, \mu_N(dy)} \right\|,
	\end{align*}
	where we used the definition of $\Gamma$ as a multi-head attention layer. We further decompose term $I$:
	\begin{align*}
		&\E \sup_{x \in B_R} \left\|\frac{\int e^{\langle x, Ay \rangle} Vy \, \mu(dy)}{\int e^{\langle x, Ay \rangle} \, \mu(dy)} - \frac{\int e^{\langle x, Ay \rangle} Vy \,\mu_N(dy)}{\int e^{\langle x, Ay \rangle} \, \mu_N(dy)} \right\| \\
		&\leq \E \sup_{x \in B_R} \left(\left\|\frac{\int e^{\langle x, Ay \rangle} Vy\, (\mu-\mu_N)(dy)}{\int e^{\langle x, Ay \rangle} \, \mu(dy)} \right\| + \left\| \int e^{\langle x, Ay \rangle} Vy\, \mu_N(dy) \left(\frac{1}{\int e^{\langle x, Ay \rangle} \, \mu(dy)} - \frac{1}{\int e^{\langle x, Ay \rangle} \,\mu_N(dy)} \right) \right\| \right) \\
		&\leq e^{c_A R^2} \E \sup_{x \in B_R} \left\|\int e^{\langle x, Ay \rangle} Vy (\mu-\mu_N)(dy) \right\| + c_V R e^{3 c_A R^2}\E \sup_{x \in B_R} \left| \int e^{\langle x, Ay \rangle} (\mu-\mu_N)(dy) \right|  \\
		&\eqdef I_I + I_{II}.
	\end{align*}
	We bound terms $I_I$ and $I_{II}$ using covering number estimates of appropriate function spaces. In more detail, let us define the function spaces $$\mathcal{F}_{1,i}(R) \defeq \{y \mapsto e^{\langle x, Ay \rangle} (Vy)_i\colon x \in B_R\}$$ and $$\mathcal{F}_2(R) \defeq \{y \mapsto e^{\langle x, Ay \rangle}\colon x \in B_R\}.$$
	Then by Dudley's chaining estimate, there exists a universal constant $C > 0$ such that
	\begin{align*}
		I_I \leq \frac{C e^{c_A R^2} \sqrt{d}}{\sqrt{N}} \max_{i \in [d]} \int_{0}^{c_V R e^{c_A R^2}} \sqrt{\log\Bigl(\mathcal{N}(\mathcal{F}_{1,i}(R);\epsilon, \| \slot \|_{L^{\infty}(B_R)})\Bigr)} \, d\epsilon
	\end{align*}
	and
	\begin{align*}
		I_{II} &\leq \frac{C c_V R e^{3c_A R^2}}{\sqrt{N}} \int_{0}^{e^{c_A R^2}} \sqrt{\log \Bigl(\mathcal{N}(\mathcal{F}_2(R); \epsilon, \|\slot\|_{L^{\infty}(B_R)}) \Bigr)} \,d \epsilon.
	\end{align*}
	Lemma \ref{lem-attn-covering-num} bounds the covering numbers of the functions classes $\mathcal{F}_{1,i}(R)$ and $\mathcal{F}_2(R).$ Plugging these covering number bounds into the estimates on $I_I$ and $I_{II},$ we deduce that
	\begin{align*}
		I_I &\leq \frac{Ce^{c_A R^2} d}{\sqrt{N}} \int_{0}^{c_V R e^{c_A R^2}}\sqrt{\log(1+2 c_A c_V R^3 e^{c_A R^2} \epsilon^{-1})} \,d\epsilon \\
		&= \frac{2C c_A c_V R^3 e^{2c_A R^2} d}{\sqrt{N}} \int_{0}^{\frac{1}{2} c_A^{-1} R^{-2}} \sqrt{\log(1+x^{-1})} \,dx \\
		&\leq \frac{C \sqrt{8} c_A^{1/2} c_V R^2 e^{2 c_A R^2}d}{\sqrt{N}},
	\end{align*}
	where we used Lemma \ref{lem-log-integral-bound} in the final line to bound the logarithmic integral. Similarly,
	\begin{align*}
		I_{II} &\leq \frac{C c_V R e^{3c_A R^2}d^{1/2}}{\sqrt{N}} \int_{0}^{e^{c_A R^2}} \sqrt{\log \left(1+2 c_A R^2 e^{c_A R^2} \epsilon^{-1} \right)} \, d \epsilon \\
		&= \frac{2 C c_V c_A R^3 e^{4 c_A R^2} d^{1/2}}{\sqrt{N}} \int_{0}^{\frac{1}{2} R^{-2}} \sqrt{\log(1+x^{-1})} \, dx \\
		&\leq \frac{C \sqrt{8} c_V e^{4c_A R^2}d^{1/2}}{\sqrt{N}}.
	\end{align*}
	Combining the estimates for $I_I$ and $I_{II},$ we obtain
	\begin{align*}
		I \leq 4C H c_Vc_W \max \left(c_A^{1/2} R^2 e^{2c_A R^2} d, e^{4c_A R^2} d^{1/2} \right) N^{-1/2}.
	\end{align*}
    
	To bound term $II,$ Corollary \ref{cor-mha-lipschitz} ensures that the Lipschitz constant of the mapping $x \mapsto \Gamma(\mu_N,x)$ is at most $1+H c_A c_Vc_W R^2$ with probability $1$. Therefore,
	\begin{align*}
		II &\leq (1+H c_A c_Vc_W R^2) \E \sfW_1(\eta, \eta_M) \\
		&\leq (1+H c_A c_Vc_W R^2) C_d R \left(1+\sqrt{\log(1+M^{1/d})} \right)M^{-1/d}
	\end{align*}
	where the last line follows from an application of Proposition \ref{prop: uniformquantization}.
\end{proof}

\subsection{Proof of Theorem~\ref{thm-approx-sample-tf}}
\begin{proof}
    Given $\epsilon > 0,$ Theorem 4.4 due to \citet{furuya2026universal} ensures the existence of a transformer $\mathsf{T}_{\epsilon}$ such that
    \begin{align*}
        \sup_{\mu \in K} \sfW_1(\mathsf{T}_{\epsilon}(\mu,\slot)_{\#}\mu, F(\mu)) \leq \epsilon/2.
    \end{align*}
    Then, by the triangle inequality, we have
    \begin{align*}
        \sup_{\mu \in K} \E \sfW_1(\mathsf{T}_{\epsilon}(\mu_N,\slot)_{\#}\mu_N,F(\mu)) &\leq \sup_{\mu \in K} \E\sfW_1(\mathsf{T}_{\epsilon}(\mu_N,\slot)_{\#}\mu_N,\mathsf{T}_{\epsilon}(\mu,\slot)_{\#}\mu) + \sup_{\mu \in K} \sfW_1(\mathsf{T}_{\epsilon}(\mu,\slot)_{\#}\mu, F(\mu)) \\
        &\leq \sup_{\mu \in K} \E\sfW_1(\mathsf{T}_{\epsilon}(\mu_N,\slot)_{\#}\mu_N,\mathsf{T}_{\epsilon}(\mu,\slot)_{\#}\mu) + \epsilon/2 \\
        &\leq C_{\epsilon} \log(N)^{\gamma_{\epsilon}/2}\left(1+\log(1+N^{1/d})\right)^{\gamma_{\epsilon}/2}N^{-\gamma_{\epsilon}/d} + \epsilon/2,
    \end{align*}
    where the last inequality follows from an application of Theorem \ref{thm-sample-tf}. We conclude the proof.
\end{proof}

\section{Auxiliary Lemmas}\label{subsec-proofs-auxlem}
This appendix contains auxiliary lemmas that support the proofs of the main results.
\subsection{Properties of Sub-Gaussian Measures}
We record some useful properties of sub-Gaussian measures. We begin with the following result which characterizes the compactness and uniform integrability properties of the sets $\{\sg_{\alpha,\beta}(\RR^d)\}_{\alpha \geq 1, \beta > 0}.$

\begin{lemma}\label{lem-subG-compact}
	For any $\alpha \geq 1$ and $\beta > 0,$ the set $\sg_{\alpha,\beta}(\RR^d)$ is $\sfW_1$-compact. Moreover, for any $p > 0,$ it holds that
	\begin{align*}
		\sup_{\mu \in \sg_{\alpha,\beta}(\RR^d)} \int_{\RR^d} \|x\|^p \, \mu(dx) < \infty \; \; \; \textrm{and} \; \; \; \sup_{\mu \in \sg_{\alpha,\beta}} \int_{\RR^d} e^{p\|x\|} \, \mu(dx) < \infty.
	\end{align*}
\end{lemma}
\begin{proof}
	We first prove the second assertion. For any $\mu \in \sg_{\alpha,\beta}(\RR^d),$ we have by Fubini's theorem that
	\begin{align*}
		\int \|x\|^p \, \mu(dx) &= \int_{0}^{\infty} \mu(\|x\|^p > t) \, dt \\
		&= \int_{0}^{\infty} \mu(\|x\| > t^{1/p}) \, dt \\
		&\leq \alpha \int_{0}^{\infty} \exp \left( -\frac{t^{2/p}}{\beta}\right)\, dt.
	\end{align*}
	This upper bounds $\int \|x\|^p \, \mu(dx)$ by a convergent integral which only depends on $\mu$ through its sub-Gaussian parameters. The proof of the exponential moment bounds is essentially the same.
	
	For the first assertion, we use the characterization due to \citet[Section 7.2]{villani2009optimal} that a sequence of measures $\{\mu_n\}_{n=1}^{\infty}$ converges to a measure $\mu$ in $\sfW_1$ if and only if $\mu_n \rightharpoonup \mu$ and $\int \|x\| \, \mu_n(dx) \rightarrow \int \|x\| \, \mu(dx).$ Let $\{\mu_n\}_{n \geq 1}$ be a sequence in $\sg_{\alpha,\beta}(\RR^d).$ Then clearly $\{\mu_n\}_{n \geq 1}$ is tight, so by Prokhorov's theorem (after passing to a subsequence), $\{\mu_n\}_{n \geq 1}$ has a weak limit $\mu\in\mathcal{P}(\RR^d)$. To prove that $\mu \in \sg_{\alpha,\beta}(\RR^d),$ let $B_R \subset \RR^d$ denote the closed ball of radius $R$ centered at zero. Then, since $\mu_n \rightharpoonup \mu,$ we have by the Portmanteau Theorem that
	\begin{align*}
		1-\alpha \exp \left( -\frac{R^2}{\beta} \right) \leq \limsup_{n\to\infty} \int_{B_R} \mu_n(dx) \leq \int_{B_R} \, \mu(dx),
	\end{align*}
	and therefore
	\begin{align*}
		\int_{\|x\| > R} \mu(dx) \leq \alpha \exp \left( -\frac{R^2}{\beta} \right),
	\end{align*}
	which proves that $\mu \in \sg_{\alpha,\beta}(\RR^d).$ Last, to prove that the first moments of $\{\mu_n\}_{n \geq 1}$ converge to the first moment of $\mu,$ define the functions $f_R(x) = \min(\|x\|,R).$ Then we have
	\begin{align*}
		&\left|\int \|x\| \, \mu_n(dx) - \int \|x\| \, \mu(dx) \right| \\
		&\leq \left| \int f_R(x)\, \mu_n(dx) - \int f_R(x) \, \mu(dx) \right| + \left|\int_{\|x\| > R}(\|x\| -R) \,\mu_n(dx) - \int_{\|x\| > R} (\|x\| - R) \, \mu(dx) \right| \\
		&\eqdef I + II.
	\end{align*}
	Since both the sequence elements $\{\mu_n\}_{n \geq 1}$ and the weak limit $\mu$ belong to $\sg_{\alpha,\beta}(\RR^d),$ the second term can be bounded by Fubini's theorem:
	\begin{align*}
		II &\leq 2 \sup_{\nu \in \sg_{\alpha,\beta}(\RR^d)} \int_{\|x\| > R} (\|x\| - R)\, \nu(dx) \\
		&\leq 2 \sup_{\nu \in \sg_{\alpha,\beta}(\RR^d)} \int_{R}^{\infty} \nu(\|x\| > t) \, dt \\
		&\leq 2\alpha \int_{R}^{\infty} \exp \left(-\frac{t^2}{\beta} \right)\, dt \\
		&= O( \exp \left(-\Omega(R^2) \right).
	\end{align*}
	Given $\epsilon > 0,$ we may choose $R_{\epsilon} > 0$ large enough that the above expression is no larger than $\epsilon/2$, and hence
	\begin{align*}
		\left|\int \|x\| \, \mu_n(dx) - \int \|x\| \, \mu(dx) \right| \leq \left| \int f_{R_{\epsilon}}(x)\, \mu_n(dx) - \int f_{R_{\epsilon}}(x) \, \mu(dx) \right| + \epsilon/2.
	\end{align*}
	Since $f_{R_{\epsilon}}$ is bounded and continuous and $\mu_n \rightharpoonup \mu,$ we can choose $n_{\epsilon} \in \mathbb{N}$ depending on $R_{\epsilon}$ such that
	\begin{align*}
		\left| \int f_{R_{\epsilon}}(x)\mu_n(dx) - \int f_{R_{\epsilon}}(x) \, \mu(dx) \right| < \epsilon/2\qfa n \geq n_{\epsilon}.
	\end{align*}
	This ensures that the first moment of $\mu_n$ is within $\epsilon$ of the first moment of $\mu$ for all $n \geq n_{\epsilon}.$
\end{proof}

Another important property, which plays a crucial role in the results of Section \ref{subsec-results-subG}, is that the set $\sg(\RR^d)$ is closed under certain pushforwards.

\begin{lemma}\label{lem-subG-pushforward}
	Fix $\alpha \geq 1$ and $\beta > 0.$ If $\mu \in \sg_{\alpha,\beta}(\RR^d)$ and $T\colon \RR^d \rightarrow \RR^k$ satisfies $\|T(x)\| \leq a+b\|x\|$ for positive constants $a$ and $b$, then $T_{\#}\mu \in \sg_{\widetilde{\alpha},\widetilde{\beta}}(\RR^k),$ where $\widetilde{\alpha} = \alpha e^{a^2/(b^2 \beta)}$ and $\widetilde{\beta} = 2b^2 \beta$.
\end{lemma}
\begin{proof}
	By our assumptions, we have $\{\|T(x)\| > R\} \subseteq \{|x| > (R-a)/b\}$. Then
	\begin{align*}
		\int_{\|x\| > R} (T_{\#}\mu)(dx) &= \int_{\|T(x)\| > R} \, \mu(dx) \\
		&\leq \int_{\|x\| > (R-a)/b} \, \mu(dx) \\
		&\leq \alpha \exp \left(-\frac{(R-a)^2}{b^2 \beta} \right) \\
		&= \alpha \exp \left( \frac{-a^2}{b^2 \beta}\right) \exp \left(\frac{-R^2 + 2a R}{b^2 \beta} \right) \\
		&\leq \alpha \exp \left( \frac{-a^2}{b^2 \beta}\right) \exp \left(\frac{-(1/2)R^2 + 2a^2}{b^2 \beta} \right) \\
		&= \widetilde{C}_1 \exp \left(-\frac{R^2}{\widetilde{C}_2} \right),
	\end{align*}
	where the final inequality is an application of Young's inequality.
\end{proof}

The following result characterizes sub-Gaussian sets in terms of Gaussian moment bounds.

\begin{lemma}\label{lem-subG-equiv-defn}
	Let $\alpha \geq 1$ and $\beta > 0$. Let $\mu \in \mathcal{P}(\RR^d).$
	\begin{enumerate}
		\item If $\int_{\RR^d} \exp \left( \frac{\|x\|^2}{\beta} \right) \, \mu(dx) \leq \alpha,$ then $\mu \in \sg_{\alpha,\beta}(\RR^d).$
		\item Conversely, if $\mu \in \sg_{\alpha,\beta},$ then 
		\begin{align*}
			\int_{\RR^d} \exp \left(\frac{\|x\|^2}{M} \right) \, \mu(dx) \leq 1+\frac{\alpha\beta}{M-\beta}\qfa  M > \beta.
		\end{align*}
	\end{enumerate}
\end{lemma}
\begin{proof}
	The first assertion follows from the standard technique of applying Markov's inequality to the exponential moment:
	\begin{align*}
		\mu \left(\|x\| > R \right) &= \mu \left(\exp \left(\frac{\|x\|^2}{\beta} \right) > \exp \left(\frac{R^2}{\beta} \right) \right) \\
		&\leq \left(\int \exp \left(\frac{\|x\|^2}{\beta} \right) \, \mu(dx) \right) \exp \left(-\frac{R^2}{\beta} \right) \\
		&\leq \alpha \exp \left(-\frac{R^2}{\beta} \right).
	\end{align*}
	For the second assertion, by the layer cake formula with derivatives, we have
	\begin{align*}
		\int \exp \left(\frac{\|x\|^2}{M} \right) \, \mu(dx) &= 1 + \frac{2}{M} \int_{0}^{\infty} t \exp \left( \frac{t^2}{M} \right) \mu(\|x\| > t) \,dt \\
		&\leq 1 + \frac{2\alpha}{M} \int_{0}^{\infty} t \exp \left(-\left(\beta^{-1}-M^{-1} \right)t^2 \right) \, dt \\
		&= 1 + \frac{\alpha}{M(\beta^{-1}-M^{-1})} \int_{0}^{\infty} e^{-t} \,dt \\
		&= 1 + \frac{\alpha}{M(\beta^{-1}-M^{-1})}\\
        &= 1 + \frac{\alpha\beta}{M-\beta}.
	\end{align*}
    The proof is complete.
\end{proof}
The following result shows that $\mathcal{P}(B_R(\RR^d))$ is contained in $\sg_{\alpha,\beta}(\RR^d)$ for appropriate choices of $\alpha$ and $\beta$.
\begin{lemma}\label{lem-compact-subg-embed}
	Fix $R > 0.$ For any $\beta > 0,$ it holds that
	\begin{align*}
		\mathcal{P}\bigl(B_R(\RR^d)\bigr) \subset \sg_{e^{R^2/\beta}, \beta}(\RR^d).
	\end{align*}
\end{lemma}
\begin{proof}
	Since
	\begin{align*}
		\sup_{\mu \in \mathcal{P}(B_R)} \int \exp \left( \frac{\|x\|^2}{\beta} \right) \, \mu(dx) \leq e^{R^2/\beta},
	\end{align*}
	the result follows from Lemma \ref{lem-subG-equiv-defn}.
\end{proof}

The following standard lemma controls the expected maximum of sub-Gaussians.

\begin{lemma}\label{lem-expected-subG-max}
	Fix $\alpha \geq 1$ and $\beta > 0.$ There exists a constant $C_{\alpha,\beta}>0$ such that
	\begin{align*}
		\sup_{\mu \in \sg_{\alpha,\beta}(\RR^d)} \E_{(x_1, \dots, x_N) \sim \mu^{\otimes N}}\left[\max_{n \in [N]} \|x_n\|^2\right] \leq C_{\alpha,\beta} \bigl(1+\log(N)\bigr).
	\end{align*}
\end{lemma}
\begin{proof}
	By Lemma \ref{lem-subG-equiv-defn}, if $\mu \in \sg_{\alpha,\beta}(\RR^d),$ then
	\begin{align*}
		\int \exp \left(\frac{\|x\|^2}{2\beta} \right) \, \mu(dx) \leq 1+\alpha.
	\end{align*}
	For $\mu \in \sg_{\alpha,\beta}(\RR^d),$ we then estimate
	\begin{align*}
		\E\left[\max_{i \in [N]}\|x_i\|^2\right] &\leq 2\beta \log \E\left[\max_{i \in [N]} e^{\|x_i\|^2/(2\beta)}\right] \\
		&\leq 2\beta \log \E\left[\sum_{i=1}^{N} e^{\|x_i\|^2/(2\beta)} \right] \\
		&\leq 2\beta \left(\log(N) + \log(1 + \alpha) \right) \\
		&\leq C_{\alpha,\beta}(1+\log(N)).
	\end{align*}
    The proof is complete.
\end{proof}

The following technical lemma controls the difference in Gaussian moments between $\mu \in \sg_{\alpha,\beta}(\RR^d)$ and its normalized restriction to a ball of radius $R.$

\begin{lemma}\label{lem-subG-Hoeffdingtruncation}
	Let $\mu \in \sg_{\alpha,\beta}(\RR^d)$ and $R > \sqrt{\beta \log(\alpha)}$. Let $\mu_R$ denote the conditional distribution of $\mu$ given $\|x\| \leq R.$ Then, for any $k \geq 2,$ it holds that
	\begin{align*}
		\left|\int e^{\|x\|^2/(k\beta)} \, \mu(dx) - \int e^{\|x\|^2/(k\beta)} \,\mu_R(dx) \right| \leq \alpha \frac{ \left(1+\frac{2}{k^2 \beta^2} \right) e^{-R^2/(k\beta)}+\left( 1 + \frac{\alpha}{k-1} \right) e^{-R^2/\beta}}{1-\alpha e^{-R^2/\beta}} .
	\end{align*}
\end{lemma}
\begin{proof}
	We have
	\begin{align*}
		&\left|\int e^{\|x\|^2/(k\beta)} \, \mu(dx) - \int e^{\|x\|^2/(k\beta)} \,\mu_R(dx) \right| = \left|\int e^{\|x\|^2/(k\beta)} \, \mu(dx) - \frac{\int_{\|x\| \leq R} e^{\|x\|^2/(k\beta)} \, \mu(dx)}{\int_{\|x\| \leq R} \mu(dx)} \right|
		\\
		&\leq \left|\frac{\int_{\|x\| > R} e^{\|x\|^2/(k\beta)} \, \mu(dx)}{\int_{\|x\| \leq R} \mu(dx)} \right| + \left|1 - \frac{1}{\int_{\|x\| \leq R} \mu(dx)} \right| \slot \int e^{\|x\|^2/(k\beta)} \, \mu(dx) \\
		&\eqdef I + II.
	\end{align*}
	For term $I,$ we upper bound the numerator and lower bound the denominator. For the numerator, we use the layer cake formula with derivative:
	\begin{align*}
		\int_{\|x\| > R} e^{\|x\|^2/(k\beta)} \, \mu(dx) &= e^{R^2/(k\beta)} \mu(\|x\| > R) + \frac{2}{k \beta} \int_{R}^{\infty} t e^{t^2/(k\beta)} \mu(\|x\| > t) \,dt \\
		&\leq \alpha e^{-\frac{(k-1)R^2}{k\beta}} + \frac{2 \alpha}{k \beta} \int_{R}^{\infty} t e^{-t^2/(k\beta)} \, dt \\
		&= \alpha e^{-\frac{(k-1)R^2}{k\beta}} + \frac{\alpha}{k \beta} \int_{R^2}^{\infty} e^{-t/(k\beta)} \,dt \\
		&= \alpha e^{-\frac{(k-1)R^2}{k \beta}} + \frac{2\alpha}{k^2 \beta^2} e^{-R^2/(k\beta)} \\
		&= \alpha \left(1+\frac{2}{k^2 \beta^2} \right) e^{-R^2/(k\beta)},
	\end{align*}
	where the last line uses the fact that $k \geq 2,$ and hence $e^{-\frac{(k-1)R^2}{k\beta}} < e^{-\frac{R^2}{k\beta}}.$ For the denominator, we have the bound 
	\begin{align*}
		\int_{\|x\| \leq R} \mu(dx) \geq 1 - \alpha e^{-R^2/\beta},
	\end{align*}
	simply due to the fact that $\mu \in \sg_{\alpha,\beta}(\RR^d).$ This establishes the bound
	\begin{align*}
		I \leq \frac{\alpha \left(1+\frac{2}{k^2 \beta^2} \right) e^{-R^2/(k\beta)}}{ 1 - \alpha e^{-R^2/\beta}}.
	\end{align*}
	For term $II,$ we use the lower bound on $\int_{\|x\| \leq R} \mu(dx)$ to deduce that
	\begin{align*}
		\left|1 - \frac{1}{\int_{\|x\| \leq R} \mu(dx)} \right| \leq \left|1 - \frac{1}{1-\alpha e^{-R^2/\beta}} \right| = \frac{\alpha e^{-R^2/\beta}}{1-\alpha e^{-R^2/\beta}}.
	\end{align*}
	Next, since $\mu \in \sg_{\alpha,\beta}(\RR^d),$ Lemma \ref{lem-subG-equiv-defn} furnishes the exponential moment bound
	\begin{align*}
		\int e^{\|x\|^2/(k\beta)} \, \mu(dx) \leq 1 + \frac{\alpha}{k-1}.
	\end{align*}
	Putting things together, we get
	\begin{align*}
		II \leq \frac{\alpha\left( 1 + \frac{\alpha}{k-1} \right) e^{-R^2/\beta}}{1-\alpha e^{-R^2/\beta}},
	\end{align*}
	and combining the bounds for terms $I$ and $II$ yields the estimate as stated in the lemma.
\end{proof}

The following provides high-probability bounds on the sub-Gaussian parameters of an empirical measure $\mu_N$ when $\mu \in \sg_{\alpha,\beta}(\RR^d)$, which may be a result of independent interest.

\begin{proposition}\label{prop-subG-empirical}
	Let $\alpha \geq 1$ and $\beta > 0,$ and let $\mu \in \sg_{\alpha, \beta}(\RR^d).$ Fix $\rho > 1$ and $k > 2\rho.$ If $x_1, \dots, x_N$ are iid samples from $\mu$ and $\mu_N = \frac{1}{N} \sum_{i=1}^{N} \delta_{x_i},$ then there exists $N_{\alpha,\beta,\rho,k} \in \mathbb{N}$ such that
	\begin{align*}
		\mathbb{P}_{x_1, \dots, x_N}\Bigl(\mu_N \in \sg_{(3+\alpha/(k-1)),k\beta}(\RR^d)\Bigr) \geq 1 - 2\alpha N^{1-\rho}\qfa N \geq N_{\alpha,\beta,\rho,k}.
	\end{align*}
\end{proposition}
\begin{proof}
	For $R > 0$, let $\mathcal{E}_{N,R}$ denote the event that $\max_{i \in [N]} \|x_i\| \leq R.$ Since $\mu \in \sg_{\alpha,\beta}(\RR^d),$ it follows from an application of the union bound that
	\begin{align*}
		\mathbb{P}_{x_1, \dots, x_N} \left( \mathcal{E}_{N,R} \right) \geq 1 -\alpha N \exp \left(-\frac{R^2}{\beta} \right).
	\end{align*}
	In particular, choosing $R = R_{N} = \sqrt{\beta \rho \log(N)},$ we have
	\begin{align*}
		\mathbb{P}_{x_1, \dots, x_N} \left( \mathcal{E}_{N,R} \right) \geq 1-\alpha N^{1-\rho}.
	\end{align*}
	To bound the desired probability, we use Lemma \ref{lem-subG-equiv-defn}, which states that if $\nu \in \mathcal{P}(\RR^d)$ and $\int e^{\frac{\|x\|^2}{k\beta}} \nu(dx) \leq 2 + \frac{\alpha}{k-1},$ then $\nu \in \sg_{(2+\alpha/(k-1)),k\beta}(\RR^d).$ Applying this to the empirical measure $\mu_N,$ it implies that
	\begin{align*}
		\mathbb{P}_{x_1, \dots, x_N}\bigl(\mu_N \in \sg_{(2+\alpha/(k-1)),k\beta}\bigr) \geq  \mathbb{P}_{x_1, \dots, x_N}\left(\frac{1}{N} \sum_{i=1}^{N} e^{\frac{\|x_i\|^2}{k\beta}} \leq 2+\frac{\alpha}{k-1}\right).
	\end{align*}
	To bound the above probability, we use the fact that
	\begin{align*}
		\mathbb{P}_{x_1, \dots, x_N}\left(\frac{1}{N} \sum_{i=1}^{N} e^{\frac{\|x_i\|^2}{k\beta}} \leq 2+\frac{\alpha}{k-1}\right) &\geq \mathbb{P}_{x_1, \dots, x_N}\left(\left\{\frac{1}{N} \sum_{i=1}^{N} e^{\frac{\|x_i\|^2}{k\beta}} \leq 2+\frac{\alpha}{k-1} \right\} \bigcap \mathcal{E}_{N,R} \right) \\
		&= \mathbb{P}_{x_1, \dots, x_N}(\mathcal{E}_{N,R})  \mathbb{P}_{x_1, \dots, x_N}\left(\frac{1}{N} \sum_{i=1}^{N} e^{\frac{\|x_i\|^2}{k\beta}} \leq 2+\frac{\alpha}{k-1} \,\bigg|\, \mathcal{E}_{N,R}\right) \\
		&\geq \left(1-\alpha N^{1-\rho} \right) \mathbb{P}_{x_1, \dots, x_N}\left(\frac{1}{N} \sum_{i=1}^{N} e^{\frac{\|x_i\|^2}{k\beta}} \leq 2+\frac{\alpha}{k-1} \,\bigg|\, \mathcal{E}_{N,R}\right).
	\end{align*}
	It therefore suffices to bound the conditional probability above. Let $\mu_R$ denote the normalized restriction of $\mu$ to the ball of radius $R,$ and note that the joint law of $x_1, \dots, x_N$ conditioned on $\mathcal{E}_{N,R}$ coincides with the joint law of $N$ iid samples from $\mu_R.$ Therefore, as a consequence of Hoeffding's inequality, we have
	\begin{align*}
		\mathbb{P}_{x_1, \dots, x_N} \left( \left| \frac{1}{N} \sum_{i=1}^{N} e^{\frac{\|x_i\|^2}{k\beta}} - \int e^{\frac{\|x\|^2}{k\beta}}\,\mu_R(dx) \right| > t \,\bigg|\, \mathcal{E}_{N,R} \right) &= \mathbb{P}_{x_1, \dots, x_N \sim \mu_R} \left( \left| \frac{1}{N} \sum_{i=1}^{N} e^{\frac{\|x_i\|^2}{k\beta}} - \int e^{\frac{\|x\|^2}{k\beta}}\,\mu_R(dx) \right| > t \right) \\
		&\leq 2 \exp \left(-\frac{2Nt^2}{\exp \left(\frac{2R^2}{k\beta} \right)} \right) \\
		&= 2 \exp \left(-2N^{1-2\rho/k} t^2 \right),
	\end{align*}
	where we used the definition of $R = R_N.$ Note that $1-\frac{2\rho}{k} > 0$ by assumption. Since $\mu \in \sg_{\alpha,\beta}(\RR^d),$ Lemma \ref{lem-subG-equiv-defn} guarantees that
	\begin{align*}
		\int e^{\frac{\|x\|^2}{k\beta}} \, \mu(dx) \leq 1+\frac{\alpha}{k-1}.
	\end{align*}
	In addition, Lemma \ref{lem-subG-Hoeffdingtruncation} proves, provided $R > \sqrt{\beta\log(\alpha)},$ that
	\begin{align*}
		\int e^{\frac{\|x\|^2}{k\beta}} \,\mu_R(dx) &\leq \int e^{\frac{\|x\|^2}{k\beta}} \, \mu(dx) + \alpha \frac{ \left(1+\frac{2}{k^2 \beta^2} \right) e^{-R^2/(k\beta)}+\left( 1 + \frac{\alpha}{k-1} \right) e^{-R^2/\beta}}{1-\alpha e^{-R^2/\beta}}.
	\end{align*}
	Recalling the definition of $R = R_N$, the second term on the right-hand side of the preceding display converges to zero as $N \rightarrow \infty$. Therefore, there exists a natural number $N_{\alpha,\beta,\rho,k}$ such that for all $N \geq N_{\alpha,\beta,\rho,k},$ we have $R > \sqrt{\beta \log (\alpha)}$ and 
	\begin{align*}
		\int e^{\frac{\|x\|^2}{k\beta}} \,\mu_R(dx) \leq 2+\frac{\alpha}{k-1}\quad\text{for all}\quad N \geq N_{\alpha,\beta}.
	\end{align*}
	Thus, setting $t = 1$ in Hoeffding's inequality, we have whenever $N \geq N_{\alpha,\beta,\rho,k}$ that
	\begin{align*}
		\mathbb{P}_{x_1, \dots, x_N}\left( \frac{1}{N} \sum_{i=1}^{N} e^{\frac{\|x_i\|^2}{k\beta}} > 3 + \frac{\alpha}{k-1} \Bigg| \mathcal{E}_{N,R} \right) &\leq \mathbb{P}_{x_1, \dots, x_N}\left( \frac{1}{N} \sum_{i=1}^{N} e^{\frac{\|x_i\|^2}{k\beta}} > 1+ \int e^{\frac{\|x\|^2}{k\beta}}\,\mu_R(dx)  \Bigg| \mathcal{E}_{N,R} \right) \\
		&\leq  \mathbb{P}_{x_1, \dots, x_N \sim \mu_R} \left( \left| \frac{1}{N} \sum_{i=1}^{N} e^{\frac{\|x_i\|^2}{k\beta}} - \int e^{\frac{\|x\|^2}{k\beta}}\,\mu_R(dx) \right| > 1 \right) \\
		&\leq 2 \exp \left(-2N^{1-2\rho/k} \right).
	\end{align*}
	Putting all the pieces together, we have 
	\begin{align*}
		\mathbb{P}_{x_1, \dots, x_N}\bigl(\mu_N \in \sg_{(3+\frac{\alpha}{k-1}),k\beta}\bigr) \geq \left(1-\alpha N^{1-\rho} \right) \left(1-2\exp \left(-2N^{1-2\rho /k} \right) \right).
	\end{align*}
	To conclude, we choose $N_{\alpha,\beta,\rho,k}$ large enough that
	\begin{align*}
		\left(1-\alpha N^{1-\rho} \right) \left(1-2\exp \left(-2N^{1-2\rho/k} \right) \right) \geq 1 - 2\alpha N^{1-\rho}\quad\text{for all}\quad  N \geq N_{\alpha,\beta,\rho,k}.
	\end{align*}
    The proof is complete.
\end{proof}

\subsection{Gaussian and Sub-Gaussian Tail Bounds}
We herein record relevant several tail estimates for Gaussian and sub-Gaussian distributions. These estimates are used several times in establishing the H\"older continuity of transformers. We begin with some elementary tail bounds for the one-dimensional Gaussian.

\begin{lemma}\label{lem: basic-gaussian-integrals}
	Let $R > 0.$ Then
	\begin{align*}
		\int_{R}^{\infty} e^{-x^2} dx &\leq \frac{1}{2R} e^{-R^2}, \\
		\int_{R}^{\infty} xe^{-x^2} dx &= \frac{1}{2} e^{-R^2},\qa \\
		\int_{R}^{\infty} x^2 e^{-x^2} dx &\leq \left(\frac{R}{2}+\frac{1}{4R} \right)e^{-R^2}.
	\end{align*}
\end{lemma}
\begin{proof}
	The second line follows from direct integration. The first line follows from the second line:
	\begin{align*}
		\int_{R}^{\infty} e^{-x^2} dx \leq \int_{R}^{\infty} \frac{x}{R} e^{-x^2} dx = \frac{1}{2R} e^{-R^2}.
	\end{align*}
	The third line follows from the first line and integration by parts:
	\begin{align*}
		\int_{R}^{\infty} x^2 e^{-x^2} dx &= \frac{1}{2} Re^{-R^2} + \frac{1}{2} \int_{R}^{\infty} e^{-x^2} dx \\
		&\leq \left(\frac{R}{2}+\frac{1}{4R} \right)e^{-R^2}.
	\end{align*}
        The proof is complete.
\end{proof}

\begin{lemma}\label{lem:integral_gaussian}
	Let $a,b > 0$ and $R > \frac{a}{2b}$. Then
	\begin{align*}
		&\int_{R}^{\infty} e^{a x - b x^2} dx \leq \frac{1}{2bR - a} e^{-b R^2 + a R}, \\
		&\int_{R}^{\infty} xe^{a x - b x^2} dx \leq \left(\frac{a}{2b^{2}R - ab} + b^{-1} \right) e^{-b R^2 + a R},\qa \\
		&\int_{R}^{\infty} x^2 e^{a x - b x^2} dx \leq \left(\frac{a^2}{4b^{5/2}\left(b^{1/2}R - a/(2b^{1/2}) \right)} + \frac{a}{2b^2} + \frac{b^{1/2}R-a/(2b^{1/2})}{2b^{3/2}} + \frac{1}{4b^{3/2}(b^{1/2}R-a/(2b^{1/2}))}\right)e^{-b R^2 + a R}.
	\end{align*}
\end{lemma}
\begin{proof}
	For the first line, we have
	\begin{align*}
		\int_{R}^{\infty} e^{a x - b x^2} dx &= e^{\frac{a^2}{4b}} \int_{R}^{\infty} e^{-\left(b^{1/2}x - \frac{a}{2b^{1/2}} \right)^2} dx \\
		&= b^{-1/2} e^{\frac{a^2}{4b}} \int_{b^{1/2} R - \frac{a}{2b^{1/2}}}^{\infty} e^{-x^2} dx \\
		&\leq \frac{b^{-1/2}}{2(b^{1/2} R - a/(2b^{1/2}))} e^{\frac{a^2}{4b} - \left(b^{1/2}R - a/(2b^{1/2}) \right)^2} \\
		&= \frac{1}{2b R - a} e^{-b R^2 + a R},
	\end{align*}
	where we used Lemma \ref{lem: basic-gaussian-integrals}. For the second line, using a similar argument, we have
	\begin{align*}
		\int_{R}^{\infty} x e^{a x - b x^2} dx &= b^{-1} e^{\frac{a^2}{4b}} \int_{b^{1/2} R - \frac{a}{2b^{1/2}}}^{\infty} \left(x + \frac{a}{2\sqrt{b}} \right) e^{-x^2} dx \\
		&\leq \left(\frac{a}{2b^{2}R - ab} + b^{-1} \right) e^{-b R^2 + a R},
	\end{align*}
	where we again used Lemma \ref{lem: basic-gaussian-integrals}. Finally, for the third line, we use the same change of variables, along with the inequalities in the first to second lines of the statement of the result:
	\begin{align*}
		&\int_{R}^{\infty} x^2 e^{a x - b x^2} dx
		= b^{-3/2} e^{\frac{a^2}{4b}} \int_{b^{1/2} R - \frac{a}{2b^{1/2}}}^{\infty} \left(\frac{a^2}{4b} + \frac{a}{b^{1/2}} x + x^2 \right) e^{-x^2} \\
		&\qquad \qquad\leq \left(\frac{a^2}{4b^{5/2}\left(b^{1/2}R - a/(2b^{1/2}) \right)} + \frac{a}{2b^2} + \frac{b^{1/2}R-a/(2b^{1/2})}{2b^{3/2}} + \frac{1}{4b^{3/2}(b^{1/2}R-a/(2b^{1/2}))}\right)e^{-b R^2 + a R},
	\end{align*}
	where we used Lemma \ref{lem: basic-gaussian-integrals} yet again.
\end{proof}
We apply these one-dimensional Gaussian tail bounds to prove tail bounds for certain moments of sub-Gaussian measures. Our first result shows that truncation is uniformly accurate.

\begin{lemma}\label{lem:trunc_unif_acc_w1}
	Let $\Pi_R\colon\RR^d\to B_R\subset \RR^d$ be the truncation map. Then
	\begin{align*}
		\sup_{\mu\in\sg_{\alpha,\beta}(\RR^d)}\sfW_1\bigl(\mu,(\Pi_R)\push\mu\bigr)\leq \frac{\alpha\beta}{2R}e^{-R^2/\beta}.
	\end{align*}
\end{lemma}
\begin{proof}
	For any $\mu\in\sg_{\alpha,\beta}$, the quantity $\sfW_1(\mu,(\Pi_R)\push\mu)$ is bounded above by
	\begin{align*}
		\int_{\RR^d}\norm{x-\Pi_R(x)}\,\mu(dx)
		= \int_{\norm{x}>R}(\norm{x}-R)\,\mu(dx)\leq\int_{R}^\infty\mu(\norm{x}>t)dt\leq \alpha \int_R^\infty e^{-t^2/\beta}\,dt\,.
	\end{align*}
	Application of Lemma~\ref{lem:integral_gaussian} with $r=R$, $a = 0$, and $b = 1/\beta$ completes the proof.
\end{proof}

We need a similar bound for quantitative sub-Gaussian uniform integrability of first moments.
\begin{lemma}\label{lem:unif_int_m1}
	For any $R>0$, it holds that
	\begin{align*}
		\sup_{\mu\in\sg_{\al,\beta}(\RR^d)}\int_{\norm{x}>R}\bigl(1+\norm{x}\bigr)\,\mu(dx)\leq\al\biggl(\frac{\beta}{2R}+1+R\biggr)e^{-R^2/\beta}.
	\end{align*}
\end{lemma}
\begin{proof}
	By the proof of Lemma~\ref{lem:trunc_unif_acc_w1}, we have
	\begin{align*}
		\int_{\norm{x}>R}\bigl(1+\norm{x}\bigr)\,\mu(dx)&=\int_{\norm{x}>R}\bigl(1+R\bigr)\,\mu(dx) + \int_{\norm{x}>R}\bigl(\norm{x}-R\bigr)\,\mu(dx)\\
		&\leq \al(1+R)e^{-R^2/\beta} + \frac{\alpha\beta}{2R}e^{-R^2/\beta},
	\end{align*}
	which is the assertion.
\end{proof}

Finally, we prove tail bounds for modulated exponential moments of sub-Gaussian measures.
\begin{lemma}\label{lem:ui_exponential_moments}
	For all $\alpha \geq 1,$ $\beta > 0$, $\mu \in \sg_{\alpha,\beta}(\RR^d)$, $\lambda > 0$, and $R > 0$, it holds that
	\begin{subequations}
		\begin{align*}
			\int_{\|y\| > R} \lambda e^{\lambda \|y\|}\bigl(\|y\|-R\bigr)\, \mu(dy) &\leq \al\lambda e^{\beta\lambda^2/2}\biggl(\frac{\beta}{R}+\lambda\frac{\beta^2}{R^2}\biggr)e^{-\frac{R^2}{2\beta}}\qa \\
			\int_{\|y\| > R} \lambda e^{\lambda \|y\|}\|y\| \bigl(\|y\| - R\bigr)\, \mu(dy) &\leq \al\lambda e^{\beta\lambda^2/2}\biggl(\lambda\frac{\beta^2}{R}+2\beta\biggr)e^{-\frac{R^2}{2\beta}}\,.
		\end{align*}
	\end{subequations}
\end{lemma}
\begin{proof}
	Fix $\mu \in \sg_{\alpha,\beta}(\RR^d)$ throughout the proof. Write $I_1$ and $I_2$ for the first and second integrals, respectively. The layer cake formula states that if $f\colon \RR \rightarrow \RR$ is nonnegative and differentiable, then
	\begin{align*}
		\int_{\|y\| > R} f(\|y\|)\, \mu(dy) = f(R) \mu(\|y\| > R) + \int_{R}^{\infty} f'(t) \mu(\|y\| > t) dt\,.
	\end{align*}
	Applying this to $f(t) \defeq \lambda e^{\lambda t}(t-R)$, we get $f'(t)=\lambda(1+\lambda(t-R))e^{\lambda t}$, $f(R)=0$, and
	\begin{align*}
		I_1 \leq \al\int_R^\infty f'(t)e^{-t^2/\beta} dt=\al\lambda \int_R^\infty \bigl(1+\lambda(t-R)\bigr)e^{\lambda t-t^2/\beta} dt\,.
	\end{align*}
	Young's inequality shows that
	\begin{align*}
		\lambda t = (\sqrt{\beta} \lambda)(t/\sqrt{\beta})\leq\frac{\beta\lambda^2}{2}+\frac{t^2}{2\beta}\,.
	\end{align*}
	Subtracting $t^2/\beta$ and exponentiating the preceding display yields
	\begin{align*}
		I_1\leq \al\lambda e^{\frac{\beta\lambda^2}{2}} \int_R^\infty \bigl(1+\lambda(t-R)\bigr)e^{-\frac{t^2}{2\beta}} dt\,.
	\end{align*}
	Now since $t^2=(R+(t-R))^2\geq R^2+2R(t-R)$, we have
	\begin{align*}
		\int_R^\infty (t-R)e^{-\frac{t^2}{2\beta}} dt\leq e^{\frac{-R^2}{2\beta}}\int_R^\infty(t-R)e^{-R(t-R)/\beta} dt=\frac{\beta^2}{R^2}e^{\frac{-R^2}{2\beta}}\int_0^\infty ue^{-u}du=\frac{\beta^2}{R^2}e^{\frac{-R^2}{2\beta}}\,.
	\end{align*}
	The first equality is due to the variable substitution $u=R(t-R)/\beta$ and the second is due to integration by parts.
	The preceding display and Lemma~\ref{lem:integral_gaussian} imply the asserted upper bound on $I_1$.
	
	For $I_2$, let $f(t)\defeq \lambda e^{\lambda t}t(t-R)$ so that $f'(t)=\lambda(\lambda t(t-R) + (2t-R))e^{\lambda t}$ and $f(R)=0$. The layer cake formula and Young's inequality yield
	\begin{align*}
		I_2\leq \al\lambda e^{\frac{\beta\lambda^2}{2}} \int_R^\infty \bigl(\lambda t(t-R) + (2t-R)\bigr)e^{-\frac{t^2}{2\beta}} dt\,.
	\end{align*}
	Integration by parts shows that
	\begin{align*}
		\int_R^\infty t(t-R) e^{-\frac{t^2}{2\beta}} dt= \beta \int_R^\infty e^{-\frac{t^2}{2\beta}}dt\,.
	\end{align*}
	This, bounding $(2t-R)\leq 2t$, and Lemma~\ref{lem:integral_gaussian} imply the asserted upper bound on $I_2$.
\end{proof}

\subsection{Properties of Transformers}
We record some prior results on the regularity of self-attention mappings. The result below controls the Lipschitz constants of in-context mappings defined by self-attention layers. These results are proven by \citet{furuya2026approximation} when $A = V$, but we record the full proofs for posterity and to make the constants explicit.

\begin{lemma}\label{lem-tf-lipschitz}
	Let $G(\mu,x)$ be defined by
	\begin{align*}
		G(\mu,x) = W\frac{\int e^{\langle x, Ay \rangle}Vy \, \mu(dy)}{\int e^{\langle x,Ay \rangle} \, \mu(dy)},
	\end{align*}
	where $A \in \RR^{d \times d}$, $V \in \RR^{k \times d},$ and $W \in \RR^{d \times k}$. Then for any $R>0$ and $R' > 0,$ it holds that
	\begin{align*}
		\sup_{\mu \in \mathcal{P}(B_R), x \in \RR^d} \|G(\mu,x)\| &\leq \|W\| \|V\| R, \\
		\sup_{\mu \in \mathcal{P}(B_R)} \operatorname{Lip}\bigl(G(\mu, \slot)\bigr) &\leq \|A\| \|V\|\|W\| R^2,\quad\text{and} \\
		\sup_{x \in B_{R'}} \operatorname{Lip}\left( G(\slot, x) \big|_{\mathcal{P}(B_R)} \right) &\leq e^{4\|A\| RR'} \max\bigl(\|V\|\|W\|(1+\|A\|R R'), \|A\| \|V\|\|W\| R R'\bigr).
	\end{align*}
\end{lemma}
\begin{proof}
	For the first estimate, fix $\mu \in \mathcal{P}(B_R)$ and $x \in \RR^d.$ We can express $G(\mu,x) = \E_{y \sim\mu_x}[Vy],$ where $\mu_x$ is defined by
	\begin{align*}
		\frac{d\mu_x}{d\mu}(y) \propto e^{\langle x, Ay \rangle}.
	\end{align*}
	In particular, $\mu_x$ is supported in $B_R$ (since $\mu$ is supported in $B_R$), which yields
	\begin{align*}
		\|G(\mu,x)\| = \|\E_{y \sim \mu_x} [WVy]\| \leq \|W\| \|V\| \E_{y \sim \mu_x}[\|y\|] \leq \|W\| \|V\| R.
	\end{align*}
	For the second estimate, viewing $G(\mu,x)$ again as a conditional expectation, a direct computation yields
	\begin{align*}
		\nabla_x G(\mu,x) = WV \operatorname{Cov}(\mu_x) A^\top.
	\end{align*}
	Since $\mu_x$ is supported in $B_R,$ we have $\|\Sigma_x\| \leq R^2.$ Therefore,
	\begin{align*}
		\sup_{x \in \RR^d} \|\nabla_x G(\mu,x)\| \leq \|W\|\|V\| \|\operatorname{Cov}(\mu_x)\| \|A^\top\| \leq \|A\| \|W\| \|V\| R^2
	\end{align*}
    holds uniformly in $\mu\in\cP(B_R)$.
    
	For the third estimate, it holds that
	\begin{align*}
		&\|G(\mu,x)-G(\nu,x)\| \leq \|W\| \frac{\left\|\int e^{\langle x, Ay \rangle} Vy \,d(\mu-\nu) \right\|}{\left| \int e^{\langle x,Ay \rangle} \,d\mu\right|} + \|W\| \left\| \int e^{\langle x, Ay \rangle} Vy \,d\nu \right\| \left|\frac{1}{\left|\int e^{\langle x, Ay \rangle} \,d\mu \right|} - \frac{1}{\left|\int e^{\langle x, Ay \rangle} \,d\nu \right|} \right| \\
		&\leq \|W\| e^{\|A\| R R'} \left\|\int e^{\langle x, Ay \rangle} Vy \,d(\mu-\nu) \right\| + \|W\| \|V\|R e^{\|A\| R R'}  \left|\frac{1}{\left|\int e^{\langle x, Ay \rangle} \, d\mu \right|} - \frac{1}{\left|\int e^{\langle x, Ay \rangle} \, d\nu \right|} \right|,
	\end{align*}
	where we used the facts that  $x \in B_{R'}$ and both $\mu$ and $\nu$ belong to $\mathcal{P}(B_R)$.  For the first term in the last line of the preceding display, a direct computation of the Jacobian of the mapping $y \mapsto e^{\langle x, Ay \rangle} Vy$ shows that
	\begin{align*}
		\operatorname{Lip} \left((y\mapsto e^{\langle x, Ay \rangle} Vy) \big|_{B_R} \right) \leq e^{\|A\| R R'} \|V\| (1+\|A\|R R'),
	\end{align*}
	and therefore, by the dual formulation of the $\sfW_1$ metric,
	\begin{align*}
		\left\|\int e^{\langle x, Ay \rangle} Vy \, d(\mu-\nu) \right\| &\leq e^{\|A\| R R'} \|V\| (1+\|A\|R R') \sfW_1(\mu,\nu).
	\end{align*}
	For the second term, using the fact that the Jacobian of $y \mapsto e^{\langle x, Ay \rangle}$ is bounded by $\|A\| R' e^{\|A\| RR'}$ for $x \in B_{R'},$ we have the bounds
	\begin{align*}
		\left|\frac{1}{\left|\int e^{\langle x, Ay \rangle} \,d\mu \right|} - \frac{1}{\left|\int e^{\langle x, Ay \rangle} \,d\nu \right|} \right| &\leq e^{2\|A\| R R'} \left|\int e^{\langle x, Ay \rangle} \,d(\mu-\nu) \right| \\
		&\leq \|A\| R' e^{3 \|A\| R R'} \sfW_1(\mu,\nu).
	\end{align*}
	Putting the pieces together, we obtain
	\begin{align*}
		\|G(\mu,x) - G(\nu,x)\| &\leq \left( e^{2 \|A\| R R'} \|W\|\|V\|(1+\|A\|R R') + \|A\| \|W\| \|V\| R R' e^{4 \|A\| R R'} \right) \sfW_1(\mu,\nu) \\
		&\leq e^{4\|A\| RR'} \max\bigl(\|W\| \|V\|(1+\|A\|R R'), \|A\| \|W\| \|V\| R R'\bigr)\, \sfW_1(\mu,\nu),
	\end{align*}
	as was to be shown.
\end{proof}

We now present a straightforward generalization of the previous result to multi-head self-attention layers with residual connections.

\begin{corollary}\label{cor-mha-lipschitz}
	Define
	\begin{align}\label{eqn:gamma-mha-lipschitz}
		(\mu,x)\mapsto \Gamma(\mu,x) \defeq x + \sum_{h=1}^{H} W_h \frac{\int e^{\langle x, A_h y \rangle} V_h y \, \mu(dy)}{\int e^{\langle x, A_h y \rangle}\, \mu(dy)} 
	\end{align}
	where $A_1, \dots, A_H \in \RR^{d \times d}$, $V_1, \dots, V_H \in \RR^{k \times d},$ and $W_1, \dots, W_h \in \RR^{d \times k}.$ Define $c_A \defeq \max_{h \in [H]} \|A_h\|$, $c_V = \max_{h \in [H]} \|V_h\|$, and $c_W = \max_{h \in [H]} \|W_h\|.$ Then for any $R>0$ and $ R' > 0,$ it holds that
	\begin{align*}
		\sup_{\mu \in \mathcal{P}(B_R)} \operatorname{Lip}\bigl(\Gamma(\mu,\slot)\bigr) \leq 1 + H c_A c_V c_W R^2\qa \\
		\sup_{x \in B_{R'}} \operatorname{Lip}\left(\Gamma(\slot,x) \big|_{\mathcal{P}(B_R)} \right) \leq H e^{4c_A RR'} \max\bigl(c_Vc_W(1+c_A R R'), c_A c_Vc_W R R'\bigr).
	\end{align*}
\end{corollary}
Via simple coupling arguments, these estimates on the Lipschitz constants of the in-context mappings defined by multi-head self-attention layers translate directly to estimates on the Lipschitz constant of the corresponding multi-head self-attention layers themselves when viewed as measure-to-measure operators.

\begin{lemma}\label{lem-tf-mtm-lipschitz}
    Instate the hypotheses of Corollary~\ref{cor-mha-lipschitz}. Let $\Gamma$ be as in \eqref{eqn:gamma-mha-lipschitz}. There exists a constant $C = C(c_A,c_V,c_W,H)>0$ such that for any $R > 0,$ it holds that
	\begin{align*}
		\sup_{\mu, \nu \in \mathcal{P}(B_R)\colon \sfW_1(\mu,\nu) \leq \epsilon} \sfW_1\bigl(\Gamma(\mu,\slot)_{\#}\mu, \Gamma(\nu,\slot)_{\#}\nu\bigr) \leq C e^{5c_A R^2} \epsilon.
	\end{align*}
\end{lemma}
\begin{proof}
	By the triangle inequality and standard Wasserstein stability estimates, it holds that
	\begin{align*}
		\sfW_1(\Gamma(\mu,\slot)_{\#}\mu,\Gamma(\nu,\slot)_{\#}\nu) &\leq \sfW_1(\Gamma(\mu,\slot)_{\#}\mu, \Gamma(\nu,\slot)_{\#}\mu) + \sfW_1(\Gamma(\nu,\slot)_{\#}\mu, \Gamma(\nu,\slot)_{\#}\nu) \\
		&\leq \left( \sup_{x\in B_R}\operatorname{Lip}\left( \Gamma(\slot,x) \big|_{\mathcal{P}(B_R)} \right) + \operatorname{Lip}\bigl(\Gamma(\nu,\slot)\bigr) \right) \sfW_1(\mu,\nu).
	\end{align*}
    Applying the Lipschitz constant estimates from Corollary \ref{cor-mha-lipschitz}, the preceding display shows that
	\begin{align*}
    \operatorname{Lip}\bigl(\mu\mapsto\Gamma(\mu,\slot)\push\mu\bigr)&\leq
		\left( \sup_{x\in B_R}\operatorname{Lip}\left( \Gamma(\slot,x) \big|_{\mathcal{P}(B_R)} \right) + \operatorname{Lip}\bigl(\Gamma(\nu,\slot)\bigr) \right)\\
        &\leq \left(1 + H c_A c_Vc_W R^2 + H e^{4 c_A R^2} \max\bigl(c_Vc_W(1+c_A R^2), c_Ac_Vc_W R\bigr) \right).
	\end{align*}
	To conclude the proof, we choose $C = C(c_A,c_V,H) > 0$ sufficiently large so that
	\begin{align*}
		\left(1 + H c_A c_Vc_W R^2 + H e^{4 c_A R^2} \max(c_Vc_W(1+c_A R^2), c_Ac_Vc_W R) \right) \leq C e^{5 c_A R^2}
	\end{align*}
	for all $R > 0.$
\end{proof}

For technical reasons, we will need estimates on the quantity
\begin{align*}
	\max_{i \in [N]} \|\mathsf{T}(\mu_N,x_i)\|,
\end{align*}
where $\mu_N = \frac{1}{N} \sum_{i=1}^{N} \delta_{x_i}$ is a discrete measure and $\mathsf{T}$ is a mean-field transformer. Such estimates are furnished in the following lemma.

\begin{lemma}\label{lem-empirical-tf-bound}
	Let $\mathsf{T}\colon \mathcal{P}(\RR^d) \times \RR^d \rightarrow \RR^k$ be a mean-field transformer. Then there exists a constant $C_{\sfT}$ such that for every $N \in \mathbb{N}$ and every measure $\nu_N$ supported on a finite set of points $\{x_1, \dots, x_N\}$, it holds that
	\begin{align*}
		\max_{i \in [N]} \|\sfT(\nu_N,x_i)\| \leq C_{\sfT} \left(\max_{n \in [N]} \|x_n\| \right).
	\end{align*}
\end{lemma}
\begin{proof}
	We will prove the result by induction on the number $L$ of layers of $\sfT$, where we view a single layer of the transformer as an application of a mean-field multi-head self-attention operator followed by a context-free MLP layer with ReLU activation
	\begin{align}\label{eq-single-tf-layer}
		(\mu,x) \mapsto \phi \bigl(\Gamma(\mu,x) \bigr),
	\end{align}
	and where we recall that the mean-field multi-headed self-attention operator with parameter matrices $A_1, \dots, A_H \in \RR^{d \times d}$, $V_1, \dots, V_H \in \RR^{k \times d}$, and $W_1, \dots, W_H \in \RR^{d \times k}$ is defined by
	\begin{align*}
		\Gamma(\mu,x) = x + \sum_{h=1}^{H}W_h \frac{\int e^{\langle x, A_hy \rangle} V_hy \, \mu(dy)}{\int e^{\langle x, A_hy \rangle} \, \mu(dy)}.
	\end{align*}
	For the base case, we assume that $\sfT(\mu,x)$ is given by \eqref{eq-single-tf-layer} for some choice of $A_1, \dots, A_H, V_1, \dots, V_H$, and $\phi.$ First, since $\phi$ is a ReLU MLP layer, it is globally Lipschitz. Hence
	\begin{align*}
		\sup_{i \in [N]} \|\phi \left(\Gamma(\nu_N,x_i) \right)\| &\leq (\|\phi(0)\| + \operatorname{Lip}(\phi)) \sup_{i \in [N]} \| \Gamma(\nu_N,x_i)\| \\
		&\leq (\|\phi(0)\| + \operatorname{Lip}(\phi)) \left(\max_{i \in [N]} \|x_i\| + \max_{h \in [H]} \|W_h\| \slot \max_{i \in [N]} \sum_{h=1}^{H} \left\|\frac{\int e^{\langle x_i, A_hy \rangle} V_hy \,\mu_N(dy)}{\int e^{\langle x_i, A_hy \rangle} \, \mu_N(dy)} \right\| \right).
	\end{align*}
	Next, we note that for each $h \in [H],$ the quantity
	\begin{align*}
		\frac{\int e^{\langle x_i, A_hy \rangle} V_hy \,\mu_N(dy)}{\int e^{\langle x_i, A_hy \rangle} \,\mu_N(dy)}
	\end{align*}
	belongs to the convex hull of $Vx_1, \dots, Vx_N.$ Therefore,
	\begin{align*}
		\max_{i \in [N]} \left\| \frac{\int e^{\langle x_i, A_hy \rangle} V_hy \,\mu_N(dy)}{\int e^{\langle x_i, A_hy \rangle} \, \mu_N(dy)} \right\| \leq \|V_h\|  \max_{i \in [N]} \|x_i\|.
	\end{align*}
	This proves the base case. Next, suppose that the claim has been proven for all transformers defined by $L$ compositions of \eqref{eq-single-tf-layer}. Let $\mathsf{T}\colon \mathcal{P}(\RR^d) \times \RR^d \rightarrow \RR^k$ be a composition of $(L+1)$ applications of \eqref{eq-single-tf-layer}. Write
	\begin{align*}
		\mathsf{T}(\mu,x) = (\phi_{L+1} \diamond \Gamma_{L+1} \diamond G_L)(\mu,x) = \phi_{L+1}\left(\Gamma_{L+1}(\mathsf{T}_L(\mu,\slot)_{\#}\mu, \mathsf{T}_L(\mu,x)) \right),
	\end{align*}
	where $\mathsf{T}_L\colon \mathcal{P}(\RR^d) \times \RR^d \rightarrow \RR^{d_L}$ is an $L$-layer transformer, and $\phi_{L+1} \diamond \Gamma_{L+1}\colon \mathcal{P}(\RR^{d_L}) \times \RR^{d_L} \rightarrow \RR^k$ is a single transformer layer. The inductive hypothesis furnishes a constant $\widetilde{C}_{\mathsf{T}} > 0$ such that
	\begin{align*}
		\sup_{x \in \RR^d} \|\mathsf{T}_L(\nu_N,x)\| \leq \widetilde{C}_{\mathsf{T}} \max_{i \in [N]} \|x_i\|.
	\end{align*}
	In addition, we can write $\mathsf{T}_L(\nu_N)_{\#}\nu_N = \sum_{i=1}^{N} c_i \delta_{z_i},$ where $z_i = \mathsf{T}_L(\nu_N,x_i).$ Thus, a parallel argument to the one used to prove the base case yields
	\begin{align*}
		\max_{i \in [N]} \|\mathsf{T}(\nu_N,x)\| &= \sup_{x \in \RR^d} \left\|(\phi_{L+1} \circ \Gamma_{L+1})(\mathsf{T}_L(\nu_N)_{\#}\nu_N,\mathsf{T}_L(\nu_N,x_i)) \right\| \\
		&\leq \max_{i \in [N]} \left\|(\phi_{L+1} \circ \Gamma_{L+1})(\mathsf{T}_L(\nu_N)_{\#}\nu_N,z_i) \right\| \\
		&\lesssim \max_{i \in [N]} \|z_i\| \\
		&\lesssim \max_{i \in [N]} \|x_i\|,
	\end{align*}
	where the inequality in the third line follows from the proof of the base case, and the inequality in the final line follows from the induction hypothesis. We conclude the proof.
\end{proof}
Lemma \ref{lem-empirical-tf-bound} guarantees that when a transformer is applied to an empirical measure $\nu_N$, the in-context map is uniformly bounded by the radius of the support of $\nu_N,$ up to multiplicative constants which naturally may depend exponentially on the number of layers.

\subsection{Wasserstein Law of Large Numbers}
One variant of the Wasserstein law of large numbers asserts that $$\lim_{N \to \infty} \E \sfW_1(\mu,\mu_N) = 0,$$ where $\mu \in \mathcal{P}(\RR^d)$ has finite first moment and $\mu_N$ is its associated empirical measure. For every $ \sfW_1 $-compact set $K \subset \mathcal{P}_1(\RR^d),$ one can actually prove that
$$ \lim_{N \to \infty} \sup_{\mu \in K} \E \sfW_1(\mu,\mu_N) = 0.
$$
Moreover, for more structured compact sets $ K $, the \emph{rate} of convergence in the previous display can be quantified. The following is a well-known quantitative Wasserstein law of large numbers for measures supported in a ball of radius $R$ (which is a $ \sfW_1 $-compact set). The proof, based on a chaining argument, is included for posterity.

\begin{proposition}\label{prop: uniformquantization}
	Let $d \in \mathbb{N}$ with $d \geq 3,$ let $\Omega \subset \RR^d$ be a compact set, and let $R \defeq \sup_{x \in \Omega} \|x\|.$ Then
	\begin{align*}
		\sup_{\mu \in \mathcal{P}(\Omega)} \E \sfW_1(\mu,\mu_N) \leq C_d R \left(1+\sqrt{\log(1+N^{1/d})} \right)N^{-1/d}\qfa N\in\N,
	\end{align*}
	where $C_d$ is a constant depending only on $d$ and $\mu_N=\frac{1}{N}\sum_{n=1}^N\delta_{X_n}$ for $X_n\sim\mu$ iid.
\end{proposition}
\begin{proof}
	We will prove the result for $\Omega = B_R,$ but the proof can easily be adapted to the case of general $\Omega.$ Let $\operatorname{Lip}_1(R) = \set{f\colon B_R \rightarrow \RR}{\textrm{$f$ is $1$-Lipschitz, $f(0) = 0$}}$. Then 
	\begin{align*}
		\E \sfW_1(\mu,\mu_N) &= \E \left[ \sup_{f \in \operatorname{Lip}_1(R)} \frac{1}{N}\sum_{i=1}^{N} \left(f(X_i) - \E_{\mu}[f] \right) \right] \\
		&\leq 2 \E\, \E_{\{\sigma_1,\ldots,\sigma_N\}\sim(\operatorname{Unif}\{1,-1\})^{\otimes N}} \left[\sup_{f \in \operatorname{Lip}_1(R)}  \frac{1}{N}\sum_{i=1}^{N} \sigma_i f(X_i) \right] \\
		&\leq 2 \inf_{\tau > 0}\left( \tau + \frac{12}{\sqrt{N}} \int_{\tau}^{2R} \sqrt{\log \mathcal{N} \left(\operatorname{Lip}_1(R), \|\slot\|_{\infty}, \epsilon \right)} \,d\epsilon \right),
	\end{align*}
	where we applied the symmetrization identity for Rademacher complexities \citep{shalev2014understanding} and Dudley's chaining bound \citep{dudley2018real}. Note that the diameter bound in the integral follows from the fact that $f\in \operatorname{Lip}_1(R)$ maps $B_R$ into $B_R.$ Next, we need to estimate the covering number of the class of $1$-Lipschitz functions. \Citet{dudley2018real} proves that
	\begin{align*}
		\log \mathcal{N} \left(\operatorname{Lip}_1(R), \|\slot\|_{\infty}, \epsilon \right) \lesssim \left(1+\frac{R}{\epsilon} \right)^d \log(1+R/\epsilon),
	\end{align*}
	where the implicit constant is dimension-dependent.  We can then upper bound the integral as
	\begin{align*}
		\int_{\tau}^{2R} \sqrt{\log \mathcal{N} \left(\operatorname{Lip}_1(R), \|\slot\|_{\infty}, \epsilon \right)} \,d\epsilon &\lesssim \int_{\tau}^{R}  \sqrt{ \left(1+\frac{R}{\epsilon} \right)^d \log(1+R/\epsilon)} \,d\epsilon \\
		&\leq \sqrt{\log(1+R/\tau)} \int_{\tau}^{R}\left(1 + \frac{R}{\epsilon} \right)^{d/2} \,d\epsilon \\
		&\lesssim \sqrt{\log(1+R/\tau)} \left((R-\tau) + R^{d/2}\int_{\tau}^{R} \epsilon^{-d/2} \,d\epsilon \right) \\
		&\leq \sqrt{\log(1+R/\tau)} \left( (R-\tau) + \frac{2}{d-2} R^{d/2} \tau^{(2-d)/2}\right) \\
		&\lesssim \sqrt{\log(1+R/\tau)} R^{d/2} \tau^{(2-d)/2}.
	\end{align*}
	Thus, taking $\tau = RN^{-1/d},$ we get the bound
	\begin{align*}
		\E \sfW_1(\mu,\mu_N) \lesssim R \left(1+\sqrt{\log(1+N^{1/d})} \right)N^{-1/d},
	\end{align*}
	where the implicit constants are independent of $R$ and depend on $\mu$ only through the dimension $d.$ This proves that
	\begin{align*}
		\sup_{\mu \in \mathcal{P}(B_R)} \E \sfW_1(\mu,\mu_N) \lesssim R \left(1+\sqrt{\log(1+N^{1/d})} \right)N^{-1/d}.
	\end{align*}
        The proof is complete.
\end{proof}

Below, we extend the Wasserstein law of large numbers from $\mathcal{P}(B_R)$ to $\sg_{\alpha,\beta}(\RR^d)$ using a truncation argument. 

\begin{proposition}\label{prop-subG-convergence}
	Fix $d \in \mathbb{N}$ with $d \geq 3.$ There exists a constant $C_{d,\alpha,\beta}$ depending only on $d,$ $\alpha,$ and $\beta$ such that
	\begin{align*}
		&\sup_{\mu \in \sg_{\alpha,\beta}(\RR^d)} \E \sfW_1(\mu,\mu_N) \leq C_{d,\alpha,\beta} \sqrt{\log(N)\left(1+\log(1+N^{1/d}) \right)}N^{-1/d} \qfa N\in\N.
	\end{align*}
\end{proposition}
\begin{proof}
	For $R > 0$ and $\mu \in \mathcal{P}(\RR^d),$ we let $\Pi_R\colon \RR^d \rightarrow B_R\subset \RR^d$ denote the truncation operator to $B_R$ and $\mu_R = (\Pi_R)_{\#}\mu.$ Given $\mu \in \sg_{\alpha,\beta}(\RR^d)$ and iid samples $X_1, \dots, X_N \sim \mu,$ we define the random measures $\mu_N = \frac{1}{N} \sum_{i=1}^{N} \delta_{X_i}$ and $\mu_{N,R} = \frac{1}{N} \sum_{i=1}^{N} \delta_{\Pi_R(X_i)}.$ Then by the triangle inequality, we have
	\begin{align*}
		\E \sfW_1(\mu,\mu_N) \leq \sfW_1(\mu,\mu_R) + \E \sfW_1(\mu_R,\mu_{N,R}) + \E \sfW_1(\mu_{N,R},\mu_N).
	\end{align*}
	For the first term, we have
	\begin{align*}
		\sfW_1(\mu,\mu_R) &\leq \int \|x-\Pi_R(x)\| \, \mu(dx) \\
		&\leq \int_{\|x\| > R} \|x\| \, \mu(dx) \\
		&\leq \left(\sup_{\nu \in \sg_{\alpha,\beta}(\RR^d)} \int \|x\|^2 \, \nu(dx) \right)^{1/2} \sqrt{\alpha} \exp\left( -\frac{R^2}{2\beta} \right).
	\end{align*}
	By Lemma \ref{lem-subG-compact}, the supremum above is finite. For the second term,
	Proposition \ref{prop: uniformquantization} guarantees that
	\begin{align*}
		\E \sfW_1(\mu_R, \mu_{N,R}) \leq C_d R \left(1+\sqrt{\log(1+N^{1/d})} \right)N^{-1/d}
	\end{align*}
	for some dimension-dependent constant $C_d,$ uniformly in $\mu.$ For the third term, we have
	\begin{align*}
		\E \sfW_1(\mu_{N,R},\mu_N) &\leq \E \left[\frac{1}{N} \sum_{i=1}^{N} \|\Pi_R(X_i) - X_i\| \right] \\
		&= \int \|\Pi_R(x) - x\| \, \mu(dx) \\
		&\leq \left(\sup_{\nu \in \sg_{\alpha,\beta}(\RR^d)} \int \|x\|^2 \nu(dx) \right)^{1/2} \sqrt{\alpha} \exp\left( -\frac{R^2}{2\beta} \right)
	\end{align*}
	yet again. We therefore have that
	\begin{align*}
		\sup_{\mu \in \sg_{\alpha,\beta}(\RR^d)} \E \sfW_1(\mu,\mu_N) &\leq 2 \left(\sup_{\nu \in \sg_{\alpha,\beta}(\RR^d)} \int \|x\|^2 \, \mu(dx) \right)^{1/2} \sqrt{\alpha} \exp\left( -\frac{R^2}{2\beta} \right) \\ &+ C_d R \left(1+\sqrt{\log(1+N^{1/d})} \right)N^{-1/d}.
	\end{align*}
	By choosing $R = \sqrt{(2\beta)/d \log(N)},$ we have
	\begin{align*}
		\sup_{\mu \in \sg_{\alpha,\beta}(\RR^d)} \E \sfW_1(\mu,\mu_N) &\leq 2 \left(\sup_{\nu \in \sg_{\alpha,\beta}(\RR^d)} \int \|x\|^2 \, \mu(dx) \right)^{1/2} \sqrt{\alpha} N^{-1/d} \\
		&\qquad\qquad + C_d \sqrt{(2\beta)/d \log(N)} \left(1+\sqrt{\log(1+N^{1/d})} \right)N^{-1/d}. \\
        &\leq  C_{d,\alpha,\beta} \sqrt{\log(N)} \left( 1 + \sqrt{\log(1+N^{1/d})} \right) N^{-1/d} \\
        &\leq C_{d,\alpha,\beta} \sqrt{\log(N)\left(1+\log(1+N^{1/d}) \right)}N^{-1/d}.
	\end{align*}
	This concludes the proof.
\end{proof}

\subsection{Covering Number Estimates} 
The proof of Theorem \ref{thm-sample-cross-attn} uses estimates on the covering numbers of various function classes associated with multi-head attention layers. We first quote the standard upper bound on the metric entropy of Euclidean balls, which can be found, e.g., in work of \citet[Chapter~3]{chewi2025statistical}.

\begin{lemma}\label{lem-entropy-balls}
	For any $\epsilon > 0$ and $R > 0,$ it holds that
	\begin{align*}
		\log\Bigl( \mathcal{N}(B_R; \epsilon, \| \slot \|) \Bigr)\leq d \log (1+2R \epsilon^{-1}).
	\end{align*}
\end{lemma}

The following result contains the main upper bound used to control the entropy integrals associated with multi-head attention.

\begin{lemma}\label{lem-attn-covering-num}
	Define the function classes $$\mathcal{F}_{1,i}(R) \defeq \{y \mapsto e^{\langle x, Ay \rangle} (Vy)_i\colon x \in B_R\}$$ and $$\mathcal{F}_2(R) \defeq \{y \mapsto e^{\langle x, Ay \rangle}\colon x \in B_R\}.$$ Then, for any $\epsilon > 0,$ it holds that
	\begin{align*}
		\max_{i \in [d]} \log \mathcal{N}(\mathcal{F}_{1,i}(R); \epsilon, \|\slot\|_{L^{\infty}(B_R)}) &\leq d \log \left(1+2 c_A c_V R^3 e^{c_A R^2} \epsilon^{-1} \right)\qa \\
		\log \mathcal{N}(\mathcal{F}_{2}(R); \epsilon, \|\slot\|_{L^{\infty}(B_R)}) &\leq d \log \left(1+2 c_A R^2 e^{c_A R^2} \epsilon^{-1} \right).
	\end{align*}
\end{lemma}
\begin{proof}
	Fix $i \in [d].$ Observe that, for any $x\in B_R$, $x' \in B_R$, and $y \in B_R,$ it holds that
	\begin{align*}
		\left|e^{\langle x, Ay \rangle} (Vy)_i - e^{\langle x', Ay \rangle} (Vy)_i \right| &\leq c_V R \left|e^{\langle x, Ay \rangle} - e^{\langle x', Ay \rangle} \right|.
	\end{align*}
	The mapping $x \mapsto e^{\langle x, Ay \rangle}$ has gradient given by $e^{\langle x, Ay \rangle} Ay,$ and hence is $c_A R e^{c_A R^2}$ Lipschitz on $B_R$ when $y \in B_R.$ Along with the previous estimate, this implies that
	\begin{align*}
		\left|e^{\langle x, Ay \rangle} (Vy)_i - e^{\langle x', Ay \rangle} (Vy)_i \right| \leq c_A c_V R^2 e^{c_A R^2} \|x-x'\| \qfa (x,x',y) \in (B_R)^3 \qa i \in [d].
	\end{align*}
	It follows that a $(c_A^{-1} c_V^{-1} R^{-2} e^{-c_A R^2} )\epsilon$-covering of $B_R$ in the Euclidean norm induces an $\epsilon$ covering of $\mathcal{F}_{1,i}(R)$ in the $L^{\infty}(B_R)$ norm. Therefore,
	\begin{align*}
		\max_{i \in [d]} \log \left(\mathcal{N}(\mathcal{F}_{1,i}(R); \epsilon, \|\slot\|_{L^{\infty}(B_R)})\right) &\leq \log\Bigl(\mathcal{N}\bigl(B_R; (c_A^{-1} c_V^{-1} R^{-2} e^{-c_A R^2}) \epsilon, \| \slot \|\bigr)\Bigr) \\
		&\leq d \log \left(1+2 c_A c_V R^3 e^{c_A R^2} \epsilon^{-1} \right),
	\end{align*}
	where the final estimate follows from Lemma \ref{lem-entropy-balls}. This proves the first inequality in the statement of the lemma. The second inequality follows by an essentially analogous argument. Using the fact that $x \mapsto e^{\langle x, Ay \rangle}$ is $c_A R e^{c_A R^2}$-Lipschitz on $B_R$ uniformly in $y \in B_R,$ it holds that
	\begin{align*}
		\log \mathcal{N}(\mathcal{F}_{2}(R); \epsilon, \|\slot\|_{L^{\infty}(B_R)}) &\leq \log (B_R; c_A^{-1} R^{-1} e^{-c_A R^2} \epsilon, \| \slot \|) \\
		&\leq d\log(1+2c_A R^2 e^{c_A R^2} \epsilon^{-1}),
	\end{align*}
	where the last line follows from Lemma \ref{lem-entropy-balls}.
\end{proof}

The following elementary estimate is useful in upper bounding entropy integrals.

\begin{lemma}\label{lem-log-integral-bound}
	For any $C > 0,$ it holds that
	\begin{align*}
		\int_{0}^{C} \sqrt{\log(1+x^{-1})}\, dx \leq 2 \sqrt{C}.
	\end{align*}
\end{lemma}
\begin{proof}
	Since $1+x \leq e^x$ for $x > 0,$ it follows that
	\begin{align*}
		\int_{0}^{C} \sqrt{\log(1+x^{-1})} \,dx \leq \int_{0}^{C} x^{-1/2} \,dx = 2 \sqrt{C}.
	\end{align*}
        The proof is complete.
\end{proof}

\end{document}